\documentclass{article}

\usepackage{microtype}
\usepackage{graphicx}
\usepackage{subcaption}
\usepackage{hyperref}

\usepackage[preprint]{icml2026}

\usepackage{amsmath}
\usepackage{amssymb}
\usepackage{mathtools}
\usepackage{amsthm}

\usepackage[capitalize,noabbrev]{cleveref}

\usepackage{thmtools,thm-restate}  

\theoremstyle{plain}
\newtheorem{theorem}{Theorem}[section]
\newtheorem{proposition}[theorem]{Proposition}
\newtheorem{lemma}[theorem]{Lemma}

\theoremstyle{definition}
\newtheorem{definition}[theorem]{Definition}
\newtheorem{assumption}[theorem]{Assumption}
\theoremstyle{remark}
\newtheorem{remark}[theorem]{Remark}
\theoremstyle{example}
\newtheorem{example}[theorem]{Example}

\icmltitlerunning{Fast Sampling for Flows and Diffusions with Lazy and Point Mass Stochastic Interpolants}

\usepackage{mathtools}
\usepackage{bm}
\usepackage{bbm}
\usepackage{derivative}
\usepackage{interval}
\intervalconfig{soft open fences}
\usepackage{etoolbox}
\usepackage{xparse}
\graphicspath{{figures/}}

\NewDocumentCommand\blank{}{\cdot}

\NewDocumentCommand\NewSurroundExp{mmm}{
  \NewDocumentCommand{#1}{sm}{
    \IfBooleanTF{##1}
    {#2\ifblank{##2}{\blank}{##2}#3}
    {\mathopen{}\left#2\ifblank{##2}{\blank}{##2}\right#3\mathclose{}}
  }
}

\NewDocumentCommand\NewSurroundExpTwo{mmm}{
  \NewDocumentCommand{#1}{smm}{
    \IfBooleanTF{##1}
    {#2\ifblank{##2}{\blank}{##2},\ifblank{##3}{\blank}{##3}#3}
    {\mathopen{}\left#2\ifblank{##2}{\blank}{##2},\ifblank{##3}{\blank}{##3}\right#3\mathclose{}}
  }
}

\NewDocumentCommand\NewScalableRelOp{mm}{
  \NewDocumentCommand{#1}{s}{
    \IfBooleanTF{##1}
    {#2}
    {\middle#2}
  }
}

\let\intervalOld\interval{}
\RenewDocumentCommand\interval{somm}{
  \IfBooleanTF{#1}
  {\IfValueTF{#2}
    {\intervalOld[#2]{#3}{#4}}
    {\intervalOld{#3}{#4}}
  }
  {\IfValueTF{#2}
    {\intervalOld[scaled,#2]{#3}{#4}}
    {\intervalOld[scaled]{#3}{#4}}
  }
}

\NewSurroundExp\ps{\lparen}{\rparen}
\NewSurroundExp\bs{\lbrack}{\rbrack}

\NewSurroundExp\set{\lbrace}{\rbrace}
\NewSurroundExp\abs{\lvert}{\rvert}
\NewSurroundExp\norm{\lVert}{\rVert}
\NewSurroundExp\floor{\lfloor}{\rfloor}
\NewSurroundExp\ceil{\lceil}{\rceil}

\NewSurroundExp\quadraticvar{\langle}{\rangle}

\let\define\definel{}

\NewDocumentCommand\given{}{\mid}

\DeclareMathOperator*{\argmin}{argmin}

\NewDocumentCommand\eps{}{\varepsilon}

\NewDocumentCommand\numsetfont{}{\mathbb}

\NewDocumentCommand\reals{}{\numsetfont{R}}
\NewDocumentCommand\posreals{}{\interval[open right]{0}{\infty}}
\NewDocumentCommand\extreals{}{\overline{\numsetfont{R}}}
\NewDocumentCommand\extposreals{}{\interval{0}{\infty}}

\NewDocumentCommand\ex{}{\mathbb{E}}
{}

\NewDocumentCommand\grad{}{\nabla}

\DeclareMathOperator{\id}{\mathbf{I}}
\DeclareMathOperator{\law}{Law}
\DeclareMathOperator{\kldiv}{KL}

\NewDocumentCommand\normal{}{\mathcal{N}}
\NewDocumentCommand\standardnormal{}{\mathcal{N}(0, \id)}

\NewDocumentCommand\sample{}{\sim}

\NewDocumentCommand\unitinterval{}{\interval{0}{1}}

\usepackage{placeins} 
\usepackage{float} 

\usepackage{fancyvrb} 

\usepackage{pgfplots}  
\pgfplotsset{compat=1.18}  
\usepgfplotslibrary{groupplots}  

\usepackage{siunitx}  

\definecolor{retrocyan}{RGB}{6,212,210}
\definecolor{retroorange}{RGB}{252,142,10}

\usepackage{enumitem}  
\setlist[enumerate]{%
  leftmargin=*,
  itemsep=3pt plus 0pt minus 0pt,
  parsep=0pt plus 0pt minus 0pt,
  topsep=3pt plus 0pt minus 0pt,
  partopsep=0pt,
}

\def\zenodourl{https://zenodo.org/records/18301277?preview=1&token=eyJhbGciOiJIUzUxMiJ9.eyJpZCI6Ijk0OTkxODE4LTA3MTktNDQyOC1iNDhhLTA0NjA3MzBhODBkMiIsImRhdGEiOnt9LCJyYW5kb20iOiJkYTViNzgzYTU0MDEzMTMwNzQ0MGQxNDhjNzhmODBlNSJ9.HHBcqQqqCo-k3zUcSnbE0ydjZkfhZrzEZ71DrmcC8wmISsqZxh9Ey53FirN_yVMh6f8RHHWZcknZgdHOGhTdkg}  

\begin{document}

\twocolumn[
  \icmltitle{Fast Sampling for Flows and Diffusions \\ with Lazy and Point Mass Stochastic Interpolants}

  \icmlsetsymbol{equal}{*}

  \begin{icmlauthorlist}
    \icmlauthor{Gabriel Damsholt}{ucph}
    \icmlauthor{Jes Frellsen}{dtu}
    \icmlauthor{Susanne Ditlevsen}{ucph}
  \end{icmlauthorlist}

  \icmlaffiliation{ucph}{Department of Mathematical Sciences, University of Copenhagen, Copenhagen, Denmark}
  \icmlaffiliation{dtu}{Department of Machine Learning and Signal Processing, Technical University of Denmark, Lyngby, Denmark}

  \icmlcorrespondingauthor{Gabriel Damsholt}{gnd@math.ku.dk}

  \icmlkeywords{Machine Learning, Artificial Intelligence, Stochastic Interpolants, Flows Matching, Diffusion Models, Generative Models, Stochastic Differential Equations, Ordinary Differential Equations, Stochastic Calculus}

  \vskip 0.3in
]
\printAffiliationsAndNotice{}  

\begin{figure*}[ht]
  \centering
  \includegraphics[width=0.333\linewidth]{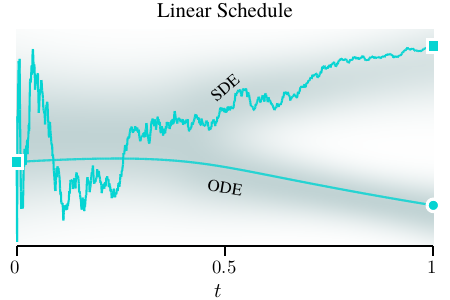}\hfill
  \includegraphics[width=0.333\linewidth]{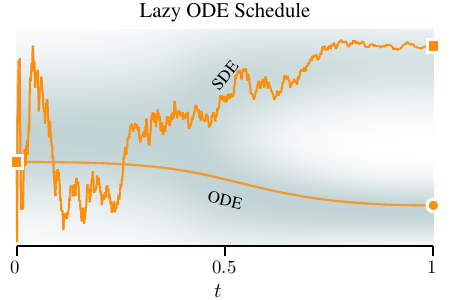}\hfill
  \includegraphics[width=0.333\linewidth]{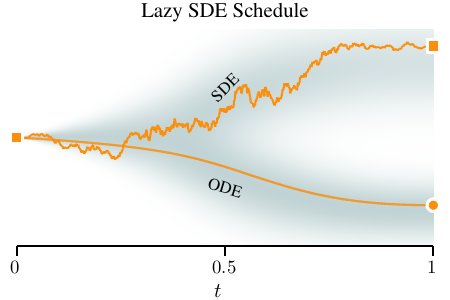}

  \caption{ODE ($\eps \equiv 0$) and statistically optimal SDE ($\eps = \eps^*$) sample path under three different interpolation schedules for $\rho_X$ a Gaussian mixture density for which the dynamics are analytically known \citep[Appendix A]{stochastic_interpolants}. All paths start in the same initial condition which sits in the initial position (cyan circle) for the leftmost two interpolants with density-admitting schedule and in the initial drift for the rightmost point mass schedule. The same Wiener process realization is shared between all three SDE solutions. Leftmost is the linear schedule from \cref{def:linear_interpolant}, middle and rightmost is the lazy interpolant for the ODE and statistically optimal SDE, respectively, as in \cref{ex:u_t_eq_t}. Since $u_t = t$ for all three schedules, all sample paths are equivalent up to a reparameterization of space. Note that all sample paths end in the same position for the ODE and SDE, respectively. See \cref{fig:schedules} for a plot of the three schedules and their statistically optimal diffusion scales.}
  \label{fig:typst_interpolant_paths}
  \vskip -0.1in
\end{figure*}

\begin{abstract}
  \emph{Stochastic interpolants} unify flows and diffusions, popular generative modeling frameworks. A primary hyperparameter in these methods is the \emph{interpolation schedule} that determines how to bridge a standard Gaussian base measure to an arbitrary target measure. We prove how to convert a sample path of a stochastic differential equation (SDE) with arbitrary diffusion coefficient under any schedule into the unique sample path under another arbitrary schedule and diffusion coefficient. We then extend the stochastic interpolant framework to admit a larger class of \emph{point mass schedules} in which the Gaussian base measure collapses to a point mass measure. Under the assumption of Gaussian data, we identify \emph{lazy} schedule families that make the drift identically zero and show that with deterministic sampling one gets a variance-preserving schedule commonly used in diffusion models, whereas with statistically optimal SDE sampling one gets our point mass schedule. Finally, to demonstrate the usefulness of our theoretical results on realistic highly non-Gaussian data, we apply our lazy schedule conversion to a state-of-the-art pretrained flow model and show that this allows for generating images in fewer steps without retraining the model.

\end{abstract}

\section{Background}

A basic goal of generative AI is to sample from a non-trivial target distribution, potentially in a very high-dimensional space, accessible only through a finite dataset of samples. Specifically, let $X^{(1)}, X^{(2)}, \ldots, X^{(N)} \in \reals^d$ be i.i.d.~samples from an unknown distribution with density $\rho_X$. We then wish to generate new samples $X \sample \rho_X$.

Various techniques for solving this basic generative problem exist. Many modern methods are based on \emph{dynamical systems}. The idea is to learn the dynamics of a dynamical system with an analytically samplable initial law, typically taken to be standard Gaussian, and with final law $\rho_X$. Samples from $\rho_X$ can then be generated by sampling from a standard Gaussian and integrating the dynamics from the initial to the final time.

We will use the flexible framework of \emph{stochastic interpolants} \cite{stochastic_interpolants} which unifies flow and diffusion based generative models. We briefly introduce this framework and then state our contributions.

\subsection{Spatially linear stochastic interpolant}

The stochastic interpolant framework allows one to couple arbitrary distributions, and even arbitrarily many of them \cite{multimarginal}, possibly in spatially non-linear ways and even using matrix coefficients \cite{multitasklearningstochasticinterpolants}. We will exclusively work with the simplest and most commonly encountered setting, namely
the ``spatially linear one-sided stochastic interpolant'' \citep[eq. 4.15]{stochastic_interpolants} which we denote
\emph{stochastic interpolant} or 
just \emph{interpolant}. Below we recall its definition and some core results.
\begin{definition}[stochastic interpolant]
    \label{def:interpolant}%
    Given independent random variables $Z, X \in \reals^d$ with $Z \sample \standardnormal, X \sample \rho_X$ and non-negative scalar functions $\alpha, \beta: \unitinterval \to \posreals$ satisfying the below criteria, the \emph{stochastic interpolant} or simply \emph{interpolant} is the stochastic process
    \[
        I_t = \alpha_t Z + \beta_t X, \quad t \in \unitinterval,
    \]
    with $\alpha, \beta \in C^1(\unitinterval)$ satisfying the boundary conditions $\alpha_0 = \beta_1 = 1, \alpha_1 = \beta_0 = 0$ and the monotonicity conditions $\dot{\alpha}_t < 0$ and $\dot{\beta}_t > 0$ for all $t \in (0, 1)$.\footnote{One can allow for the one-sided derivatives of $\alpha$ and $\beta$ to explode at either endpoint as long as the one-sided derivatives exist in the extended reals, see \citet{stochastic_interpolants} for a discussion on this. We require bounded derivatives for simplicity.}
\end{definition}

Note that $I_0 \sample \standardnormal$ and $I_1 \sample \rho_X$ for all valid choices of $\alpha, \beta$.

One of the primary objects of interest is the following.
\begin{definition}[schedule]
  The function pair $(\alpha, \beta)$ in \cref{def:interpolant} is the \emph{interpolation schedule} or simply \emph{schedule}.
\end{definition}
The schedule is a hyperparameter that determines how 
$\standardnormal$ and $\rho_X$ are bridged in time. A canonical choice is $\alpha_t = 1-t, \beta_t = t$; this is the choice taken in flow matching.

To proceed we define the following functions.
\begin{definition}{}
    \label{def:quantities}%
    For an interpolant as in \cref{def:interpolant} with density $\rho: \unitinterval \times \reals^d \rightarrow \interval[open]{0}{\infty}$ and $\eps: \unitinterval \rightarrow \posreals$ we define the following functions for $(t, x) \in \unitinterval \times \reals^d$:
    \begin{align*}
        \eta_Z (t, x) &\define \ex\bs{Z \given I_t = x}, \tag*{``noise predictor''}\\
        \eta_X (t, x) &\define \ex\bs{X \given I_t = x}, \tag*{``data predictor''}\\
        s(t, x) &\define \grad \log \rho(t, x), \tag*{``score''}\\
        b^\eps (t, x) &\define \dot{\alpha}_t \eta_Z (t, x) + \dot{\beta}_t \eta_X (t, x) + \eps_t s(t, x). \tag*{``drift''}
    \end{align*}
\end{definition}

The central idea behind the stochastic interpolant framework is to construct a family of \emph{stochastic differential equations} (SDEs) indexed by an arbitrary time-dependent additive diffusion scale $\eps: \interval{0}{1} \to \posreals$, whose solutions have the same marginal distributions (but different joint distributions) as $I_t$, as captured in the following theorem \citep[Corollary 18]{stochastic_interpolants}.
\begin{theorem}[SDE]
    \label{theorem:sde}%
    Let $\eps_t \in C^1(\interval{0}{1})$ be any non-negative scalar function. Then the solutions to the SDE
    \begin{equation}
        \label{eq:SDE}
    \odif{X^\eps_t} = b^\eps (t, X^\eps_t) \odif{t} + \sqrt{2 \eps_t} \odif{W_t},
    \end{equation}
    solved forward in time $t$ with $X_0^\eps \sample \standardnormal$ independent of $W$, satisfy $\law(X^\eps_t) = \law(I_t)$ for all $t \in \unitinterval$.
\end{theorem}

With the above result we can pick any non-negative diffusion scale $\eps$, sample $X_0^\eps \sample \standardnormal$ and solve the SDE forward in time from $t=0$ to $t=1$ provided we know the drift $b^\eps$. Then $X_1^\eps \sample \rho_X$, as desired. In particular, setting $\eps \equiv 0$ gives a deterministic probability flow ODE.

Each function in \cref{def:quantities} can be characterized as the minimizer of a simple squared-loss objective, enabling unbiased Monte-Carlo estimation from data without simulating the SDE. One then uses the learned estimate of the dynamics to generate data, e.g.~one learns $\hat{b}^\eps \approx b^\eps$ from data and solves the SDE \eqref{eq:SDE} with $b^\eps$ replaced by $\hat{b}^\eps$ to sample $\hat{X} \sample \hat{\rho}_X \approx \rho_X$.

Throughout we assume the following, the validity of which depends on the properties of $\rho_X$, the smoothness of $(\alpha, \beta)$ and the choice of $\eps$ (see \citet{stochastic_interpolants} for details).
\begin{assumption}
    \label{assum:first_assum}
    The data density satisfies $\rho_X(x) > 0$ for all $x \in \reals^d$, and the density $\rho: \unitinterval \times \reals^d \rightarrow \interval[open]{0}{\infty}$ of the interpolant satisfies $\rho(\blank, x) \in C^1(\unitinterval)$ for all $x \in \reals^d$, $\rho(t, \blank) \in C^2 (\reals^d)$ for all $t \in \unitinterval$, and similarly for the functions in \cref{def:quantities}. Furthermore, the SDE \eqref{eq:SDE} has a unique strong solution for all initial conditions.
\end{assumption}

\section{Contributions}
\begin{enumerate}
    \item In \cref{theorem:path_transform_sde} we generalize existing results from the ODE to the SDE setting and show that we can uniquely convert any strong solution to the SDE \eqref{eq:SDE} with arbitrary schedule and diffusion scale into the corresponding strong solution to the SDE \eqref{eq:SDE} with an arbitrary different schedule and diffusion scale.
    \item In \cref{def:point_mass_schedule} and \cref{theorem:point_mass_generation} we extend the interpolation schedule to include the point mass boundary condition $\alpha_0 = 0$ and give sufficient conditions for the SDE to exist in this extended setting with normally distributed initial conditions that sit in the initial drift instead of in the initial position. We stress that this result even holds in the ODE case, i.e.~when $\eps \equiv 0$.
    \item Assuming that $X \sample \standardnormal$ we prove in \cref{prop:lazy_schedule_families} that the \emph{lazy schedule family} -- the family of schedules such that the dynamics become identically zero -- for the ODE satisfies the variance-preserving property $\alpha_t^2 + \beta_t^2 = 1$, while for the SDE with statistically optimal diffusion scale it satisfies $\alpha_t^2 + \beta_t^2 = \beta_t$. In particular, $\alpha_0 = 0$ in the latter case, which motivates our schedule extension to the point mass case.
    \item In \cref{prop:ode_opt_schedule_transform,prop:sde_opt_schedule_transform,alg:ode_sample,alg:sde_sample} we derive particularly simple and theoretically well-motivated formulas for reparameterizing space (but not time) in any pretrained flow matching model with linear schedule $\overline{\alpha}_t = 1-t, \overline{\beta}_t=t$ to yield a lazy ODE and statistically optimal SDE schedule, respectively.
    \item In \cref{sec:experiments} we conduct empirical experiments with a state-of-the-art flow matching model for natural image generation and show that using this reparameterization allows for fewer-step numerical sampling in the ODE case and especially so for the statistically optimal SDE, which is otherwise numerically undefined at $t=0$.
\end{enumerate}

\section{Related work}
The stochastic interpolant framework \cite{stochastic_interpolants} unifies existing continuous-time generative paradigms based on dynamical systems, namely flow matching \cite{lipman2022flowmatching} and diffusion models\footnote{\emph{Score-based generative models}, the continuous-time equivalents of \emph{denoising diffusion probabilistic models} \cite{NEURIPS2020_4c5bcfec}.} \cite{pmlr-v37-sohl-dickstein15,NEURIPS2019_3001ef25,song2021scorebased} by 
disentangling the two primary hyperparameters: the \emph{interpolation schedule} coupling the Gaussian and data distribution in time, and the \emph{stochasticity} of the stochastic process that samples this coupling. The optimal level of stochasticity has mostly been explored in the stochastic interpolant literature \cite{stochastic_interpolants,probabilistic_forecasting,sit_interpolants,lipschitz_schedules}, whereas more work exists on tuning the interpolation schedule with the aim of easing numerical integration, albeit often heuristically or empirically based \cite{nichol2021improved,NEURIPS2022_a98846e9,NEURIPS2020_92c3b916,accel_opt_steps}. Our work is primarily inspired by \citet{lipschitz_schedules}, but by considering the statistically optimal SDE instead of just the ODE we arrive at \emph{point mass schedules} for which little work exists \cite{stochastic_interpolants,probabilistic_forecasting}. Our work also resembles \citet{align_your_steps}, but in addition to time reparameterizations we also consider space reparameterizations, even extending to point mass schedules.

Changing the interpolation schedule post-training has recently been explored for ODE sampling \cite{liu2024letusflowtogether,principlesdiffusionmodels,lipschitz_schedules}, but to the best of our knowledge pathwise equivalence has never been established for SDE sampling before.

More sophisticated approaches for sampling from flow and diffusion models in fewer steps exist; some notable examples are flow rectification \cite{liu2023flow}, consistency models \cite{pmlr-v202-song23a,boffi2025flow} and (entropy regularized) optimal transport based approaches \cite{pmlr-v202-pooladian23a,NEURIPS2023_c428adf7,bortoli2021diffusion,plug_in_bridge}. These approaches are mostly supplementary to the choice of schedule that we explore.

\section{Intra-interpolant conversion formulas}
In this section we fix the schedule $(\alpha, \beta)$ and show how to convert between the functions in \cref{def:quantities}. The proofs can be found in \cref{sec:app_proofs}.

The following diffusion scale plays an important role in our work. We define it now and motivate it later.
\begin{definition}
    \label{def:eps_optimal}%
    Given an interpolation schedule $(\alpha, \beta)$ satisfying the criteria in \cref{def:interpolant} we denote by $\eps^*_t: \unitinterval \to \extposreals$ the scalar function
    \[
        \eps_t^* = \alpha_t^2 \frac{\dot{\beta}_t}{\beta_t} - \alpha_t \dot{\alpha}_t,
    \]
    where at $t=0$ the equation is interpreted as a right-limit in the extended reals.
\end{definition}

We start by showing that for a fixed interpolation schedule $(\alpha, \beta)$, any one of the functions from \cref{def:quantities} uniquely defines all remaining functions through affine transformations that simplify using the above definition for $\eps^*$. This result generalizes the well-known result relating the score to the noise predictor \citep[Theorem 8]{stochastic_interpolants} and \citet[Proposition 6.3.1]{principlesdiffusionmodels} for the ODE case $\eps \equiv 0$.

\begin{restatable}[intra-interpolant conversion formulas]{proposition}{PropIntraConversion}
    \label{prop:intra_conversion}%
    For a fixed interpolation schedule $(\alpha, \beta)$ and non-negative $\eps \in C^1(\interval{0}{1})$, each of the functions $\eta_Z, \eta_X, s, b^\eps$ uniquely defines all of the others on $\interval[open]{0}{1} \times \reals^d$. In particular, each function can be expressed in terms of the score as
    \begin{align*}
        \eta_Z (t, x) &= -\alpha_t s(t, x), \\
        \eta_X (t, x) &= \ps{x + \alpha_t^2 s(t, x)}/\beta_t, \\
        b^\eps (t, x) &= \ps{\eps_t^* + \eps_t} s(t, x) + (\dot{\beta}_t/\beta_t) x.
    \end{align*}
\end{restatable}

The next theorem characterizes $\eps^*$ as the statistically optimal diffusion scale in a precise sense. This is essentially equation 13 in \citet{probabilistic_forecasting} but our proof is simpler.
\begin{restatable}{theorem}{TheoremEpsOptKLDiv}
    \label{theorem:eps_opt_kl_div}
    Assume we are given statistically learnt estimates $\hat{\eta}_Z, \hat{\eta}_X, \hat{s}$ or $\hat{b}^\eps$ of the true functions from \cref{def:quantities} with remaining estimates given by \cref{prop:intra_conversion}. Let $X^\eps_t$ solve the SDE \eqref{eq:SDE} and consider the plug-in SDE
    \[
        \odif{\hat{X}^\eps_t} = \hat{b}^\eps (t, \hat{X}^\eps_t) \odif{t} + \sqrt{2 \eps_t} \odif{W_t},
    \]
    solved forward in time $t$ with $\hat{X}_0^\eps \sample \standardnormal$ independent of $W$. Then the minimizer of the Kullback-Leibler divergence between the path measures of $X^\eps$ and $\hat{X}^\eps$ over the diffusion scale $\eps$ is
    \[
        \argmin_{\eps \geq 0, \eps \in C^1(\interval[open left]{0}{1})} \kldiv \ps{X^\eps, \hat{X}^\eps} = \eps^*,
    \]
    and the minimum is
    \[
        \kldiv\ps{X^{\eps^*}, \hat{X}^{\eps^*}} = \int_0^1 \eps^*_t \ex\bs{\norm{s (t, I_t) - \hat{s} (t, I_t)}^2} \odif{t}.
    \]
\end{restatable}

Unfortunately, the statistically optimal diffusion scale $\eps^*$ is ill-behaved numerically in the sense that for any schedule $(\alpha, \beta)$ satisfying \cref{def:interpolant} it is infinite and non-integrable at $t=0$ (\cref{prop:eps_opt}). Nonetheless, the \emph{statistically optimal SDE}, i.e., the SDE \eqref{eq:SDE} with $\eps = \eps^*$, is well-defined on $t \in \interval[open left]{0}{1}$ and behaves like a (non-continuous) white-noise process as $t \to 0_+$. See \cref{prop:eps_opt_gives_OU_rev_process} which proves an equivalence between $\eps^*$, the Ornstein-Uhlenbeck process, and diffusion models.

\section{Inter-interpolant conversion formulas}
We now show how to convert the functions from \cref{def:quantities} from one schedule to another and how to convert between solutions to the SDEs. The following \emph{linear schedule} used in flow matching simplifies the conversion formulas.

\begin{definition}[linear interpolant]
    \label{def:linear_interpolant}%
    Let $\overline{\alpha}_t \define 1 - t, \overline{\beta}_t \define t$. Then we denote by $\overline{I}_t \define \overline{\alpha}_t Z + \overline{\beta}_t X$ the \emph{linear interpolant}. We denote by $\overline{\eta}_Z, \overline{\eta}_X, \overline{s}, \overline{b}^{\overline{\eps}}$ its associated functions as in \cref{def:quantities} and by $\overline{X}^{\overline{\eps}}_t$ its associated SDE \eqref{eq:SDE}.
\end{definition}

\begin{restatable}{proposition}{PropLinearEpsOptimal}
    \label{prop:linear_eps_optimal}%
    For the linear interpolant $\overline{I}_t$ from \cref{def:linear_interpolant} it holds that $\overline{\eps}^*_t = (1 - t)/t$. Therefore
    \[
        \int_s^t 2 \overline{\eps}^*_u \odif{u} = 2\ps{\log(t) - t - \log(s) + s}
    \]
    for $0 < s \leq t \leq 1$ and the integral is infinite for $s=0$.
\end{restatable}

\begin{restatable}{definition}{DefUAndC}
    \label{def:u_and_c}%
    For $t \in \unitinterval$ we define the functions
    \[
        c_t \define \alpha_t + \beta_t, \quad u_t \define \beta_t / c_t.
    \]
\end{restatable}
The interpretation is that $c$ is a (generally non-invertible) space-change while $u$ is a proper time change. Indeed, the monotonicity assumptions on the schedule $(\alpha, \beta)$ in \cref{def:interpolant} imply that $u \in C^1(\unitinterval)$ and is strictly increasing on $\interval[open]{0}{1}$ so that its inverse $t_u \define u^{-1}(u)$ exists.

The following proposition can be seen as a supplement to \cref{prop:intra_conversion}. It rephrases and generalizes proposition 3.1 from \citet{lipschitz_schedules} and proposition 6.3.3 from \citet{principlesdiffusionmodels} which only consider the ODE case ($\eps \equiv 0$).
\begin{restatable}[inter-conversion formulas]{proposition}{PropInterConversion}
    \label{prop:inter_conversion}%
    For a schedule $(\alpha, \beta)$ and non-negative $\eps \in C^1(\interval{0}{1})$, the functions $\eta_Z, \eta_X, s, b^\eps$ are uniquely defined from any of the functions $\overline{\eta}_Z, \overline{\eta}_X, \overline{s}, \overline{b}^{\overline{\eps}}$ associated with the linear interpolant. In particular,
    \begin{align*}
    \eta_Z (t, x) &= \overline{\eta}_Z \ps{u_t, x/c_t}, \quad \eta_X (t, x) = \overline{\eta}_X \ps{u_t, x/c_t}, \\
    s (t, x) &= (1/c_t) \overline{s} \ps{u_t, x/c_t}, \\
    b^\eps (t, x) &= \frac{\dot{c}_t}{c_t} x + c_t \dot{u}_t \overline{b}^{\overline{\eps}} \ps{u_t, \frac{x}{c_t}} \quad \text{for} \quad \overline{\eps}_{u_t} = \frac{\alpha_t \eps_t}{\beta_t \eps^*_t}.
    \end{align*}
\end{restatable}

We finally prove our strongest conversion result, which to the best of our knowledge is new for $\eps \not \equiv 0$.
\begin{restatable}[pathwise conversion]{theorem}{TheoremPathTransformSDE}
    \label{theorem:path_transform_sde}
    Let $(\alpha, \beta)$ be a schedule satisfying the requirements in \cref{def:interpolant}, let $z \in \reals^d$ and let $(W_t)_{t \in \unitinterval}$ be a Wiener process realization. Let $X_t^\eps$ solve the SDE \eqref{eq:SDE} in physical time $t$ from $X_0^\eps = z$ driven by $W$ with diffusion scale $\eps$. Let $\overline{X}_{u}^{\overline{\eps}}$ solve the SDE \eqref{eq:SDE} with linear schedule in reparameterized time $u = u_t$ from $\overline{X}_{0}^{\overline{\eps}} = z$ driven by
    \[
        \overline{W}_u = \int_0^{t_u} \sqrt{\dot{u}_s} \odif{W_s}, \quad u \in \unitinterval
    \]
    with diffusion scale $\overline{\eps}_{u} = (\alpha_{t_u} \eps_{t_u})/(\beta_{t_u} \eps^*_{t_u})$.
    Then the solution $X^\eps_t$ can be obtained from the solution $\overline{X}^{\overline{\eps}}_{u_t}$ associated with the linear interpolant via the formula
    \[
        X^\eps_t = c_t \overline{X}^{\overline{\eps}}_{u_t}, \quad t \in \unitinterval.
    \]
    In particular, $X^\eps_1 = \overline{X}^{\overline{\eps}}_1$ (since $c_1 = 1$).
\end{restatable}
With this result, we can sample paths from the interpolant with any schedule and any diffusion scale provided we can sample paths from the linear interpolant (simply apply \cref{theorem:path_transform_sde} twice to convert from an arbitrary schedule instead of the linear schedule). This is illustrated in the two leftmost plots in \cref{fig:typst_interpolant_paths} for the ODE and statistically optimal SDE, respectively.
\begin{remark}
    We shall see that the identity time-transform $u_t = t$ is of particular interest. Then  \cref{theorem:path_transform_sde} simplifies significantly, in particular, $\overline{W} = W$.
\end{remark}

\section{Point mass schedule}
We use the preceding results to extend \cref{def:interpolant} to transport to the target distribution from a point mass instead of a (non-degenerate) Gaussian. We delay the motivation to the next section.
\begin{definition}[point mass schedule and interpolant]
    \label{def:point_mass_schedule}%
    A function pair $\alpha, \beta \in C^1(\unitinterval)$ satisfying
    \begin{enumerate}
        \item $\alpha_0 = \alpha_1 = \beta_0 = 0, \beta_1 = 1$,
        \item $\alpha_t > 0, \dot{\beta}_t > 0$ for all $t \in (0, 1)$,
        \item $\beta_t = o(\alpha_t) \text{ as } t \to 0_+$ (equivalently $u_0 = 0$),
        \item $\alpha_t^2 = O(\beta_t) \text{ as } t \to 0_+$,
        \item $\dot{u}_0 \define \lim_{t \rightarrow 0_+} \dot{u}_t < \infty$,
        \item $\odv{}{t} (\beta_t/\alpha_t) > 0$ for all $t \in \interval[open]{0}{1}$ (equivalently $\dot{u}_t > 0$),
    \end{enumerate}
    we call a \emph{point mass schedule}. The associated interpolant we call a \emph{point mass interpolant}.
\end{definition}

\begin{remark}
The condition $\alpha_0 = \beta_0 = 0$ is sufficient and necessary for the density of the interpolant to collapse to a point mass at $t=0$. The condition $u_0 = 0$ is sufficient and necessary for the initial drift to be analytically samplable; if otherwise $u_0 > 0$ we would have to sample the non-trivial law of $\overline{I}_{u_0} \neq \overline{I}_0$ to sample the initial drift, which is impossible in practice. The condition $\alpha_t^2 = O(\beta_t) \text{ as } t \to 0_+$ is technical and can be relaxed; see the remarks at the end of the proof of \cref{theorem:point_mass_generation}. The assumption $\dot{u}_0 < \infty$ gives $u \in C^1(\unitinterval)$, which is not strictly necessary but we impose it for simplicity. The last assumption is equivalent to $\dot{u}_t > 0$ for all $t \in \interval[open]{0}{1}$ so that $t \mapsto u_t$ is a proper time change.
\end{remark}

Remarkably, the initial drift of the point mass interpolant can be defined pointwise from the initial condition of the linear interpolant, even in the ODE case, as we show below.

\begin{restatable}{theorem}{TheoremPointMassGeneration}
    \label{theorem:point_mass_generation}
    Let $(\alpha, \beta)$ be a point mass schedule as in \cref{def:point_mass_schedule}. Then given any non-negative bounded diffusion scale $\eps \in C^1(\unitinterval)$, the SDE \eqref{eq:SDE} is well-defined in the classical sense for $t \in \interval[open left]{0}{1}$ with initial position $X_0^\eps = 0$, and the initial drift is well-defined in distribution. It holds that $\law\ps{X_t^\eps} = \law\ps{I_t}$ for all $t \in \unitinterval$ and \cref{theorem:path_transform_sde} still applies: in particular, the point mass SDE solution and initial drift relates to the solution and initial condition of the linear interpolant as $X_t^\eps = c_t \overline{X}^{\overline{\eps}}_{u_t}$ for $t \in \unitinterval$ and $b^\eps(0, X_0^\eps) = \ps{\dot{c}_0 - \lim_{t \to 0_+} (\eps_t / \alpha_t)} \overline{X}_0^{\overline{\eps}}$.
\end{restatable}
\begin{remark}
    The initial drift in \cref{theorem:point_mass_generation} is always bounded for the ODE case $\eps \equiv 0$, whereas for the SDE case it is bounded if and only if we choose the diffusion scale such that $\eps_t \in O(\alpha_t)$ as $t \to 0_+$.
\end{remark}

\cref{theorem:point_mass_generation} shows that while the normally distributed initial condition sits at the initial position for the density-admitting stochastic interpolant from \cref{def:interpolant}, the initial condition instead sits at the initial SDE drift (or ODE velocity when $\eps \equiv 0$) for the point mass interpolant from \cref{def:point_mass_schedule}. To the best of our knowledge such a non-standard ODE has not been considered before in the flow or diffusion literature, whereas the SDE has been treated in slightly different settings in \citet{stochastic_interpolants,probabilistic_forecasting}.

A sample path from a point mass interpolant is plotted together with its linear interpolant path in \cref{fig:typst_interpolant_paths} for the ODE and statistically optimal SDE, respectively.

We finally extend proposition 2.5 from \citet{lipschitz_schedules} to the point mass setting.
\begin{restatable}{proposition}{PropKLDivInvariance}
    \label{prop:kl_div_invariance}%
    Let $(\alpha, \beta)$ be any valid schedule, possibly a point mass schedule. Then under the assumptions in \cref{theorem:eps_opt_kl_div}, the minimal path-measure KL-divergence attained when $\eps = \eps^*$ is invariant to the choice of schedule.
\end{restatable}
\cref{prop:kl_div_invariance} dictates that we must optimize the schedule for some criterion different from statistical optimality, which is what we do in the next section.

\section{Lazy schedules under Gaussian data assumption}

In this section, we make the oversimplifying assumption that $X \sample \standardnormal$. This allows us to identify families of schedules for which the SDE drift is identically zero, the motivation being that such schedules minimize the energy required to transport from base to target measure. We call these schedules \emph{lazy}. In the experimental section, we demonstrate empirically that the results we obtain under the Gaussian data assumption are  useful even in realistic, high-dimensional and highly non-Gaussian settings.

\subsection{Relation to flows and diffusions}

To simplify the presentation we focus on two natural choices for the diffusion scale: $\eps \equiv 0$ and $\eps = \eps^*$. The former corresponds to ODE sampling and reduces to standard flow matching under the linear interpolant (see \cref{def:linear_interpolant}). The latter reduces to sampling a time-reversed Ornstein-Uhlenbeck process as in diffusion modeling (\cref{prop:eps_opt_gives_OU_rev_process}) and we will henceforth refer to this as ``statistically optimal SDE sampling''. Numerical integration is simpler for ODE sampling, but statistically optimal SDE sampling is maximally robust to approximation errors in the learned drift estimate as per \cref{theorem:eps_opt_kl_div}. Throughout we use the notation $b \define b^{\eps \equiv 0}$ for the ODE velocity and $b^* \define b^{\eps=\eps^*}$ for the statistically optimal SDE drift.

\subsection{Identifying lazy schedules}

Requiring the velocity and drift to be identically zero under the Gaussian data assumption removes one degree of freedom in the schedule.

\begin{restatable}[lazy schedule families]{proposition}{PropLazyScheduleFamilies}
    \label{prop:lazy_schedule_families}%
    Assume that $X \sample \standardnormal$. Then for all $(t, x) \in \unitinterval \times \reals^d$ it holds that
    \begin{enumerate}
        \item
        \[
            b (t, x) = 0 \quad \iff \quad \alpha_t^2 + \beta_t^2 = 1.
        \]
        In words, the ODE velocity is identically zero if and only if the schedule is variance preserving.
        \item 
        \[
            b^* (t, x) = 0 \quad \iff \quad \alpha_t^2 + \beta_t^2 = \beta_t.
        \]
        This implies $\alpha_0 = 0, \beta_t = o(\alpha_t), \alpha_t^2 = O(\beta_t) \text{ as } t \to 0_+$. We have that $\dot{u}_0 < \infty \iff \dot{\beta}_t = O(\sqrt{\beta_t})$ in which case the schedule is a point mass schedule per \cref{def:point_mass_schedule}. It also holds for all $t \in \unitinterval$ that $\eps^*_t = \dot{\beta}_t/2$ so that
        \[
            \int_s^t 2 \eps^*_u \odif{u} = \beta_t - \beta_s \quad \text{ for } \quad 0 \leq s \leq t \leq 1.
        \]
    \end{enumerate}
\end{restatable}

\cref{prop:lazy_schedule_families} curiously shows that requiring the ODE velocity to be identically zero under the Gaussian data assumption amounts to choosing a variance preserving schedule, which is typically used in diffusion models but not in flow models. On the other hand, for statistically optimal SDE sampling it instead leads to point mass schedule, which is underexplored in the literature.

\begin{remark}
    From \cref{prop:lazy_schedule_families} together with the inevitable boundary conditions $\beta_0 = 0, \beta_1 = 1$ we see that $\int_0^1 2 \eps^*_u \odif{u} = 1$, i.e.~the quadratic variation of the martingale part of solutions to our SDE from $t=0$ to $t=1$ are fixed at 1. This is remarkable since for $\alpha_0 = 1$ the same integral is fixed at $\infty$ by \cref{prop:eps_opt}. This result holds even when $X$ is not normally distributed. However, one easily shows that normalizing this quadratic variation by the total variance gives $\int_0^t (2 \eps^*_u)/\beta_u \odif{u} = [\log (\beta_u)]_0^t = \infty$ for all $t \in \interval[open left]{0}{1}$, resembling the result from \cref{prop:eps_opt}.
\end{remark}

A natural next question is how to eliminate the last degree of freedom. We give two natural examples below and then discuss why the second example is useful in practice.

\begin{example}[$\beta_t = t$]
    Setting $\beta_t = t$ gives, in the ODE case, $\alpha_t = \sqrt{1 - t^2}$. Interestingly, this is \emph{not} the usual linear schedule used in flow matching models. Also $\dot{\alpha}_1 = -\infty$. Considering instead the SDE case we get $\alpha_t = \sqrt{t(1-t)}$ and $2 \eps^*_t = 1$. When $X \sample \standardnormal$, the solution to the SDE is a standard $d$-dimensional Brownian motion $(W_t)_{t \in \unitinterval}$ with $W_0 = 0$. Also $\dot{\alpha}_0 = \infty$ and $\dot{\alpha}_1 = -\infty$. The endpoint explosion of $\dot{\alpha}$ in these schedules is worrisome.
\end{example}

\begin{example}[$u_t = t$]
    \label{ex:u_t_eq_t}%
    Define $d_t := (1-t)^2 + t^2$. Then for the ODE case requiring $u_t = t$ gives the schedule $\alpha_t = (1 - t)/\sqrt{d_t}, \beta_t = t/\sqrt{d_t}$. Then $\dot{\alpha}_0 = 0$, $\dot{\alpha}_1 = -1$ and $\dot{\beta}_0 = 1, \dot{\beta}_1 = 0$. Considering instead the SDE case gives the schedule $\alpha_t = t(1 - t)/d_t, \beta_t = t^2/d_t$ which is a point mass schedule per \cref{prop:lazy_schedule_families} since $\dot{\beta}_t = O(\sqrt{\beta_t})$ as $t \to 0_+$. We now have $\dot{\alpha}_0 = 1$, $\dot{\alpha}_1 = -1$ and $\dot{\beta}_0 = \dot{\beta}_1 = \eps^*_0 = \eps^*_1 = 0$, and also $\dot{c}_0 = 1 = \lim_{t \to 0_+} \eps_t^*/\alpha_t$ so that the initial drift is zero per \cref{theorem:point_mass_generation}. The fact that all functions are bounded everywhere in both schedules is attractive from a numerical perspective.
\end{example}

We can further characterize and motivate the constraint $u_t = t$ in \cref{ex:u_t_eq_t} as follows. Consider the log-signal-to-noise ratio $\lambda_t \define \log(\beta_t^2/\alpha_t^2)$. Thus, $\lambda_0 = -\infty$ and $\lambda_1 = \infty$.\footnote{Even if the schedule is a point mass schedule we still have the limit $\lim_{t \to 0_+} \beta_t/\alpha_t = 0$.} It might be natural to require $\lambda_t$ to be linear in \emph{logit-time}. Specifically, define $\tau(t) \define 2 \log(t/(1-t))$, with inverse $t(\tau) = (1 + \exp(-\tau/2))^{-1}$, i.e.~the logistic function with growth rate $1/2$. Then requiring $\lambda_t = \tau(t)$ is equivalent to requiring $u_t = t$.

\subsection{Lazy schedule transformation of a pretrained flow matching model}

Based on what we have shown so far, we can convert any flow or diffusion model (more generally any stochastic interpolant) with an arbitrary schedule to its corresponding lazy schedule variant (under the Gaussian assumption). The motivation is that numerically solving the ODE or SDE is easier when the dynamics are gentle. That the dynamics in practice become gentler under the lazy schedule even when we seriously violate the Gaussian data assumption is strongly suggested by the empirical results in the next section. To make the schedule conversion particularly simple, we focus on the $u_t = t$ case considered in \cref{ex:u_t_eq_t}. We focus on converting from a flow matching model velocity $v^\text{flow}$, which is identical to the linear ODE velocity $\overline{b}$ of the linear interpolant $\overline{I}_t$ from \cref{def:linear_interpolant} (see also \cref{sec:relation_to_flow_models}).

\begin{restatable}[linear velocity to lazy ODE velocity]{proposition}{PropODEOptScheduleTransform}
    \label{prop:ode_opt_schedule_transform}
    Define the schedule
    \[
        \alpha_t \define (1-t)/\sqrt{d_t}, \quad
        \beta_t \define t/\sqrt{d_t}, \quad
        d_t \define (1-t)^2 + t^2,
    \]
    as in \cref{ex:u_t_eq_t} and assume we are given access to the ODE velocity $\overline{b} = v^{\text{flow}}$ of the linear interpolant from \cref{def:linear_interpolant}. Then the velocity
    satisfies
    \[
        b(t,x) = ((1-2t)/(d_t)) x + (1/\sqrt{d_t})\overline{b}(t,\sqrt{d_t}x),
    \]
    for all $(t,x)\in \unitinterval \times \reals^d$. In particular, the initial velocity satisfies
    \[
        b(0,z) = z + \overline{b}(0, z) = \ex\bs{X} \quad \forall z \in \reals^d.
    \]
\end{restatable}

\begin{restatable}[linear velocity to lazy SDE drift]{proposition}{PropSDEOptScheduleTransform}
    \label{prop:sde_opt_schedule_transform}%
    Define the schedule
    \[
        \alpha_t \define t(1-t)/d_t, \quad
        \beta_t \define t^2/d_t, \quad
        d_t \define (1-t)^2 + t^2,
    \]
    as in \cref{ex:u_t_eq_t} and assume we are given access to the ODE velocity $\overline{b} = v^{\text{flow}}$ of the linear interpolant from \cref{def:linear_interpolant}. Then the statistically optimal SDE drift satisfies
    \[
        b^*(t,x) = (2/d_t)((1-2t)x + t \overline{b}(t,(d_t/t)x)),
    \]
    for all $(t,x) \in \interval[open left]{0}{1} \times \reals^d$, and the initial drift is identically zero, i.e.~$b^*(0,0) = 0$.
\end{restatable}

We can restate \cref{prop:ode_opt_schedule_transform,prop:sde_opt_schedule_transform} as particularly simple algorithms for how to solve the lazy ODE and SDE using the linear ODE velocity. Pseudocode for these algorithms is given in \cref{alg:ode_sample,alg:sde_sample}.

We stress that we can prove results similar to \cref{prop:ode_opt_schedule_transform,prop:sde_opt_schedule_transform} but from a pretrained diffusion model instead of flow model, or more generally from a pretrained stochastic interpolant with any schedule and any training objective. This is omitted for brevity.

\begin{remark}
    \label{remark:inperfect_flow_vel}%
    Note that \cref{prop:ode_opt_schedule_transform,prop:sde_opt_schedule_transform} generally do not hold exactly when we only have access to an approximate linear flow velocity $\hat{v}^{\text{flow}} \approx v^{\text{flow}}$ learned from data, as is always the case in practice.
\end{remark}

\section{Experiments}\label{sec:experiments}%
\begin{figure}[tb]
    \centering
    \includegraphics[width=\linewidth]{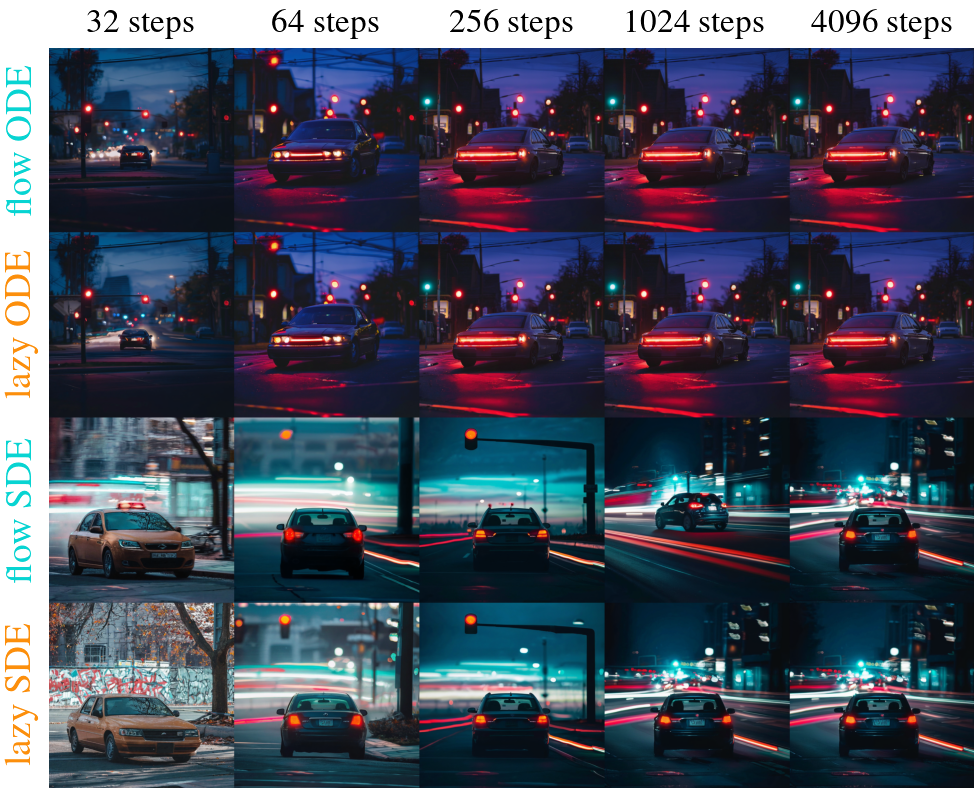}
    
    \caption{Sample images from the PRX flow model with a varying number of predictor-corrector steps for the ODE and statistically optimal SDE, respectively, sampled using the original linear \textcolor{retrocyan}{flow} model schedule versus converting to the \textcolor{retroorange}{lazy} schedule. The text prompt ``A car is stopped at a red light'' was used. As expected, the image quality improves as the number of solver steps increases and seems to converge to the ``ground truth'' reference image which depends on whether an ODE or SDE is used but less so on which schedule is used. See \href{\zenodourl}{this Zenodo link} for an animation.}
    \label{fig:convergence_samples}%
    \vskip -0.1in
\end{figure}

To investigate the usefulness of our lazy schedules in realistic, high-dimensional settings with highly non-Gaussian data we conduct experiments on the latest, largest Photoroom Experimental (PRX) text-to-image classifier-free guided latent-space flow model \cite{photoroom_prx256t2i,almazan2025prx_release,lipman2022flowmatching,ho2022cfg} which has around 1.3 billion parameters and is currently in beta\footnote{Available on HuggingFace at \url{https://huggingface.co/Photoroom/prx-1024-t2i-beta}.}. We fix the guidance strength at the recommended value of $5$ and wish to investigate whether dynamically converting to the lazy schedules using \cref{prop:ode_opt_schedule_transform,prop:sde_opt_schedule_transform} as described in detail in \cref{alg:ode_sample,alg:sde_sample} enables us to generate good images in fewer steps using the ODE and statistically optimal SDE, respectively. We use the predictor-corrector scheme which informally is a mix between the explicit Euler(-Maruyama) scheme and the Heun scheme and which we find is the most efficient of the three. See \cref{sec:app_schemes} for a precise definition of each scheme and some sample image comparisons.

\begin{figure*}[t!]
    \centering
    \includegraphics[width=0.49\linewidth]{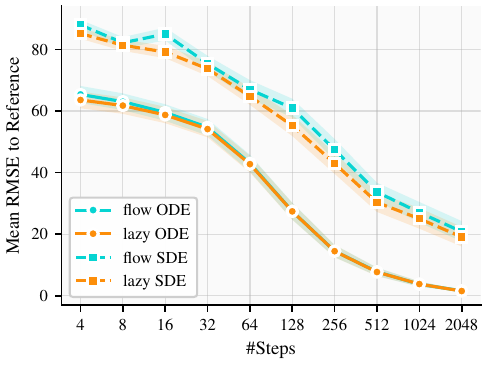}
    \includegraphics[width=0.49\linewidth]{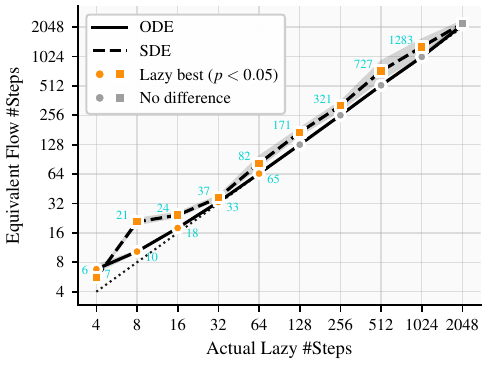}
    
    \caption{\textbf{Left:} Pixel-wise RMSE predictor-corrector convergence plot for ODE (circle markers and solid curves) and statistically optimal SDE (square markers and dashed curves) generation with the original linear \textcolor{retrocyan}{flow} model schedule versus the converted \textcolor{retroorange}{lazy} schedule as in \cref{fig:convergence_samples}. Average and 95\% confidence interval (shaded bands) is over $100$ prompts, initial conditions and Wiener process realizations. One sees that convergence is (almost) monotone as a function of the number of solver steps for all four configurations. Compare with \cref{fig:convergence_plot_within_step}. \textbf{Right:} Visual answer to the question ``Using $n_i$ numerical solver steps with the lazy schedule, how many solver steps would I on average need to use with the linear flow model schedule to achieve the same RMSE to the reference image?''. The answer for each $n_i$ on the $x$-axis is plotted on the $y$-axis (\textcolor{retrocyan}{cyan} labels indicate nearest rounded integer). Shaded bands indicate $95\%$ confidence intervals. The lazy schedule almost always performs statistically significantly better (and never worse) for both ODE and SDE generation, but the improvement is much bigger for SDE generation. All confidence intervals are calculated by bootstrapping with $\num[group-separator={,}]{10000}$ samples.}
    \label{fig:convergence_plot}%
    \vskip -0.1in
\end{figure*}

Instead of looking at standard but arbitrary goodness-of-image metrics such as Fréchet Inception distance \cite{fid} and Inception score \cite{is}, we look at how fast the predictor-corrector solver converges to the idealized ``ground truth'' reference image as a function of the total number of numerical solver steps $n_i$. Since the idealized ``ground truth'' reference image is unattainable in practice,
we define it as the average image after $n_k=4096 \approx \infty$ solver steps, i.e.,
\[
    \texttt{img}_\text{reference}^\text{ODE} \define (\texttt{img}^\text{ODE,flow}_{4096} + \texttt{img}^\text{ODE,lazy}_{4096})/2,
\]
and similarly for the SDE. This is reasonable since in the idealized setting where $\hat{v}^{\text{flow}} = v^{\text{flow}}$, i.e.~when the flow velocity is learned perfectly, then $\texttt{img}^\text{ODE,flow}_{\infty} = \texttt{img}^\text{ODE,lazy}_{\infty}$, and similarly for the SDE. Results where $\texttt{img}_\text{reference}^\text{ODE} \define \texttt{img}^\text{ODE,flow}_{4096}$ is used instead are shown in \cref{fig:convergence_plot_flow_ref}.

As text prompts we use the first $100$ captions from the COCO validation dataset \cite{coco,coco_captions} as listed in \cref{sec:app_coco_captions}. For each experiment we fix a text prompt from our dataset and randomly sample an initial condition and (for SDE generation) a Wiener process realization. We let the number of numerical solver steps vary in the range $n_1=4, n_2=8, \ldots, n_k=4096$. This generates $k=11$ images of increasing quality for each of the four configurations \texttt{(ODE,SDE)} $\times$ \texttt{(\textcolor{retrocyan}{flow},\textcolor{retroorange}{lazy})} sharing the same prompt, initial condition and Wiener process realization, as illustrated in \cref{fig:convergence_samples}. We then calculate the root mean square error (RMSE) over pixels\footnote{This corresponds to a dimension-normalized Euclidean distance in pixel-space.} between the image produced with $n_i$ numerical solver steps and the reference image, i.e.~we calculate
\[
    \texttt{RMSE}(\texttt{img}^\text{ODE,flow}_{n_i}, \texttt{img}_\text{reference}^\text{ODE})
\]
for each $i$ and similarly for $\texttt{img}^\text{ODE,lazy}_{n_i}$ and the SDE. We repeat this experiment for each of our $100$ different prompts, initial conditions and Wiener process realizations and compute the average RMSE. This measures numerical convergence of the solver under each configuration.\footnote{To be precise, this measures convergence of the latent representations \emph{after} decoding, which is more semantically meaningful than convergence in latent space.} The result is plotted in \cref{fig:convergence_plot} and animations for the generated images for all $100$ prompts are accessible at \href{\zenodourl}{this Zenodo link}.

Under the linear flow model schedule the statistically optimal SDE has infinite initial drift and diffusion scale (\cref{prop:eps_opt}) so it is numerically undefined. To deal with this we sample the first step under the lazy schedule; this is easily shown to be equivalent to setting $X^*_{\Delta t} = \sqrt{(1 - \Delta t)^2 + (\Delta t)^2} W_{\Delta t}$ where $W_{\Delta t}$ is the initial increment of the fixed Wiener process.

In \cref{fig:convergence_plot} (left) we see that each of the four configurations converges more or less monotonically to its reference image, but convergence is significantly faster with ODE generation than with SDE generation, indicating the difficulty in numerically solving SDEs. See also \cref{fig:convergence_plot_within_step} showing that the two SDEs converge more slowly to a common reference image.

To test whether the effect of converting to the lazy schedule is statistically significantly positive for ODE and SDE generation, respectively, we ask the following question: ``Using $n_i$ numerical solver steps with the lazy schedule, how many solver steps would I on average need to use with the linear flow model schedule to achieve the same RMSE to the reference image?''. The answer to this question is illustrated in \cref{fig:convergence_plot} (right). For ODE integration the improvement is statistically significant up until $n_i=64$ yet quite small; one saves at most three solver steps (which for $n_i=4$ almost halves the sampling time) using the lazy schedule instead of the original linear flow model schedule. For SDE generation the improvement is much larger; for example, using $128$ predictor-corrector steps with the lazy schedule achieves a performance (on average) similar to using $171$ solver steps with the linear flow model schedule.

\section{Discussion}
We stress that since the PRX model outputs a learned approximation $\hat{v}^{\text{flow}} \approx v^{\text{flow}}$ the flow-to-lazy conversion is inexact as noted in \cref{remark:inperfect_flow_vel}. The fact that the lazy schedule performs statistically better than the linear flow schedule that the PRX model was trained with is therefore very convincing and leads us to hypothesize that training a model from scratch using e.g.~the lazy ODE schedule from \cref{ex:u_t_eq_t} would lead to significantly bigger improvements in sampling quality.

We also stress that the slow convergence of the SDE does \emph{not} allow us to conclude that the statistically optimal SDE generates poorer images than the ODE. The more interesting metric, which is harder to measure, is \emph{distributional fit to the test dataset} and the tradeoff between that and ease of numerical integration for the ODE versus the SDE. We leave this to future work but note that the sample images in \cref{sec:app_exp_results} indicate that using a moderate amount of solver steps (say $n_i \geq 64$) for the statistically optimal SDE generates aesthetically pleasing and detailed images, especially with the lazy schedule.

On the theoretical side it would be valuable to extend the lazy schedule family beyond the Gaussian data assumption. We hypothesize that for the statistically optimal SDE one will end up with a point mass schedule under reasonable assumptions on the data density $\rho_X$.

\FloatBarrier  

\section*{Software and Data}
Code will be released upon acceptance.

\section*{Impact Statement}
This paper presents work whose goal is to advance the field of Machine
Learning. There are many potential societal consequences of our work, none
which we feel must be specifically highlighted here.

\bibliography{lazy_interpolants}

@article{stochastic_interpolants,
    author  = {Michael Albergo and Nicholas M. Boffi and Eric Vanden-Eijnden},
    title   = {Stochastic Interpolants: A Unifying Framework for Flows and Diffusions},
    journal = {Journal of Machine Learning Research},
    year    = {2025},
    volume  = {26},
    number  = {209},
    pages   = {1--80}
}

@inproceedings{probabilistic_forecasting,
    title = 	 {Probabilistic Forecasting with Stochastic Interpolants and Föllmer Processes},
    author =       {Chen, Yifan and Goldstein, Mark and Hua, Mengjian and Albergo, Michael Samuel and Boffi, Nicholas Matthew and Vanden-Eijnden, Eric},
    booktitle = 	 {Proceedings of the 41st International Conference on Machine Learning},
    pages = 	 {6728--6756},
    year = 	 {2024},
    editor = 	 {Salakhutdinov, Ruslan and Kolter, Zico and Heller, Katherine and Weller, Adrian and Oliver, Nuria and Scarlett, Jonathan and Berkenkamp, Felix},
    volume = 	 {235},
    series = 	 {Proceedings of Machine Learning Research},
    month = 	 {21--27 Jul},
    publisher =    {PMLR}
}

@inproceedings{multimarginal,
    author = {Albergo, Michael and Boffi, Nicholas and Lindsey, Michael and Vanden-Eijnden, Eric},
    booktitle = {International Conference on Learning Representations},
    editor = {B. Kim and Y. Yue and S. Chaudhuri and K. Fragkiadaki and M. Khan and Y. Sun},
    pages = {55884--55901},
    title = {Multimarginal Generative Modeling with Stochastic Interpolants},
    volume = {2024},
    year = {2024}
}

@inproceedings{multitasklearningstochasticinterpolants,
    title     = {Multitask Learning with Stochastic Interpolants},
    author    = {Hugo Negrel and Florentin Coeurdoux and Michael S. Albergo and Eric Vanden-Eijnden},
    booktitle = {Advances in Neural Information Processing Systems (NeurIPS 2025)},
    year      = {2025},
    note      = {Accepted to NeurIPS 2025 as a Spotlight; camera-ready hosted on OpenReview.}
}

@inproceedings{sit_interpolants,
    author="Ma, Nanye
    and Goldstein, Mark
    and Albergo, Michael S.
    and Boffi, Nicholas M.
    and Vanden-Eijnden, Eric
    and Xie, Saining",
    editor="Leonardis, Ale{\v{s}}
    and Ricci, Elisa
    and Roth, Stefan
    and Russakovsky, Olga
    and Sattler, Torsten
    and Varol, G{\"u}l",
    title="SiT: Exploring Flow and Diffusion-Based Generative Models with Scalable Interpolant Transformers",
    booktitle="Computer Vision -- ECCV 2024",
    year="2024",
    publisher="Springer Nature Switzerland",
    address="Cham",
    pages="23--40",
    isbn="978-3-031-72980-5"
}

@misc{lipschitz_schedules,
      title={Lipschitz-Guided Design of Interpolation Schedules in Generative Models}, 
      author={Yifan Chen and Eric Vanden-Eijnden and Jiawei Xu},
      year={2025},
      eprint={2509.01629},
      archivePrefix={arXiv},
      primaryClass={stat.ML}
}

@misc{photoroom_prx256t2i,
    author       = {{Photoroom}},
    title        = {{PRX: Open Text-to-Image Generative Model (prx-1024-t2i-beta)}},
    year         = {2025},
    howpublished = {\url{https://huggingface.co/Photoroom/prx-1024-t2i-beta}},
    note         = {Hugging Face model card. Github commit ID: 84772b42065c4ffcfc10a7269e1aebf104fc0ef1},
    urldate      = {2026-01-07}
}

@misc{almazan2025prx_release,
    author       = {Almaz{\'a}n, Jon and Bertoin, David and Frigg, Roman},
    title        = {We’re open-sourcing our text-to-image model and the process behind it},
    year         = {2025},
    month        = nov,
    howpublished = {\url{https://huggingface.co/blog/Photoroom/prx-open-source-t2i-model}},
    urldate      = {2026-01-07}
}

@misc{ho2022cfg,
    author       = {Ho, Jonathan and Salimans, Tim},
    title        = {Classifier-Free Diffusion Guidance},
    year         = {2022},
    eprint       = {2207.12598},
    archivePrefix= {arXiv},
    primaryClass = {cs.LG},
    howpublished = {\url{https://arxiv.org/abs/2207.12598}}
}

@book{LeGall2016Brownian,
    author    = {Le Gall, Jean-François},
    title     = {Brownian Motion, Martingales, and Stochastic Calculus},
    publisher = {Springer Cham},
    year      = {2016},
    edition   = {1},
    series    = {Graduate Texts in Mathematics},
    doi       = {10.1007/978-3-319-31089-3},
    isbn      = {978-3-319-31089-3}
}

@article{stein,
    author = {Charles M. Stein},
    title = {{Estimation of the Mean of a Multivariate Normal Distribution}},
    volume = {9},
    journal = {The Annals of Statistics},
    number = {6},
    publisher = {Institute of Mathematical Statistics},
    pages = {1135 -- 1151},
    year = {1981},
    doi = {10.1214/aos/1176345632}
}

@article{anderson_time_reversal,
    title   = {Reverse-time diffusion equation models},
    author  = {Anderson, Brian D.O.},
    journal = {Stochastic Processes and their Applications},
    volume  = {12},
    number  = {3},
    pages   = {313--326},
    year    = {1982},
    doi     = {10.1016/0304-4149(82)90051-5}
}

@inproceedings{NEURIPS2020_4c5bcfec,
    author = {Ho, Jonathan and Jain, Ajay and Abbeel, Pieter},
    booktitle = {Advances in Neural Information Processing Systems},
    editor = {H. Larochelle and M. Ranzato and R. Hadsell and M.F. Balcan and H. Lin},
    pages = {6840--6851},
    publisher = {Curran Associates, Inc.},
    title = {Denoising Diffusion Probabilistic Models},
    volume = {33},
    year = {2020}
}

@inproceedings{pmlr-v37-sohl-dickstein15,
    title = 	 {Deep Unsupervised Learning using Nonequilibrium Thermodynamics},
    author = 	 {Sohl-Dickstein, Jascha and Weiss, Eric and Maheswaranathan, Niru and Ganguli, Surya},
    booktitle = 	 {Proceedings of the 32nd International Conference on Machine Learning},
    pages = 	 {2256--2265},
    year = 	 {2015},
    editor = 	 {Bach, Francis and Blei, David},
    volume = 	 {37},
    series = 	 {Proceedings of Machine Learning Research},
    address = 	 {Lille, France},
    month = 	 {07--09 Jul},
    publisher =    {PMLR}
}

@inproceedings{NEURIPS2019_3001ef25,
    author = {Song, Yang and Ermon, Stefano},
    booktitle = {Advances in Neural Information Processing Systems},
    editor = {H. Wallach and H. Larochelle and A. Beygelzimer and F. d\textquotesingle Alch\'{e}-Buc and E. Fox and R. Garnett},
    pages = {},
    publisher = {Curran Associates, Inc.},
    title = {Generative Modeling by Estimating Gradients of the Data Distribution},
    volume = {32},
    year = {2019}
}

@inproceedings{
    song2021scorebased,
    title={Score-Based Generative Modeling through Stochastic Differential Equations},
    author={Yang Song and Jascha Sohl-Dickstein and Diederik P Kingma and Abhishek Kumar and Stefano Ermon and Ben Poole},
    booktitle={International Conference on Learning Representations},
    year={2021}
}

@article{principlesdiffusionmodels,
    title={The principles of diffusion models},
    author={Lai, Chieh-Hsin and Song, Yang and Kim, Dongjun and Mitsufuji, Yuki and Ermon, Stefano},
    journal={arXiv preprint arXiv:2510.21890},
    year={2025}
}

@misc{
    nichol2021improved,
    title={Improved Denoising Diffusion Probabilistic Models},
    author={Alexander Quinn Nichol and Prafulla Dhariwal},
    year={2021}
}

@inproceedings{NEURIPS2022_a98846e9,
    author = {Karras, Tero and Aittala, Miika and Aila, Timo and Laine, Samuli},
    booktitle = {Advances in Neural Information Processing Systems},
    editor = {S. Koyejo and S. Mohamed and A. Agarwal and D. Belgrave and K. Cho and A. Oh},
    pages = {26565--26577},
    publisher = {Curran Associates, Inc.},
    title = {Elucidating the Design Space of Diffusion-Based Generative Models},
    volume = {35},
    year = {2022}
}

@inproceedings{NEURIPS2020_92c3b916,
    author = {Song, Yang and Ermon, Stefano},
    booktitle = {Advances in Neural Information Processing Systems},
    editor = {H. Larochelle and M. Ranzato and R. Hadsell and M.F. Balcan and H. Lin},
    pages = {12438--12448},
    publisher = {Curran Associates, Inc.},
    title = {Improved Techniques for Training Score-Based Generative Models},
    volume = {33},
    year = {2020}
}

@INPROCEEDINGS{accel_opt_steps,
    author={Xue, Shuchen and Liu, Zhaoqiang and Chen, Fei and Zhang, Shifeng and Hu, Tianyang and Xie, Enze and Li, Zhenguo},
    booktitle={2024 IEEE/CVF Conference on Computer Vision and Pattern Recognition (CVPR)}, 
    title={Accelerating Diffusion Sampling with Optimized Time Steps}, 
    year={2024},
    volume={},
    number={},
    pages={8292-8301},
    doi={10.1109/CVPR52733.2024.00792}
}

@inproceedings{align_your_steps,
    author = {Sabour, Amirmojtaba and Fidler, Sanja and Kreis, Karsten},
    title = {Align your steps: optimizing sampling schedules in diffusion models},
    year = {2024},
    publisher = {JMLR.org},
    booktitle = {Proceedings of the 41st International Conference on Machine Learning},
    articleno = {1750},
    numpages = {29},
    location = {Vienna, Austria},
    series = {ICML'24}
}

@inproceedings{
    lipman2022flowmatching,
    title={Flow Matching for Generative Modeling},
    author={Yaron Lipman and Ricky T. Q. Chen and Heli Ben-Hamu and Maximilian Nickel and Matthew Le},
    booktitle={The Eleventh International Conference on Learning Representations},
    year={2023}
}

@misc{liu2024letusflowtogether,
    title        = {Let us Flow Together},
    author       = {Liu, Qiang},
    year         = {2024},
    month        = dec,
    note         = {Working draft (dated December 24, 2024)},
    howpublished = {\url{https://www.cs.utexas.edu/~lqiang/PDF/flow_book.pdf}},
    urldate      = {2026-01-07}
}

@inproceedings{pmlr-v202-song23a,
    title = 	 {Consistency Models},
    author =       {Song, Yang and Dhariwal, Prafulla and Chen, Mark and Sutskever, Ilya},
    booktitle = 	 {Proceedings of the 40th International Conference on Machine Learning},
    pages = 	 {32211--32252},
    year = 	 {2023},
    editor = 	 {Krause, Andreas and Brunskill, Emma and Cho, Kyunghyun and Engelhardt, Barbara and Sabato, Sivan and Scarlett, Jonathan},
    volume = 	 {202},
    series = 	 {Proceedings of Machine Learning Research},
    month = 	 {23--29 Jul},
    publisher =    {PMLR}
}

@article{
    boffi2025flow,
    title={Flow map matching with stochastic interpolants: A mathematical framework for consistency models},
    author={Nicholas Matthew Boffi and Michael Samuel Albergo and Eric Vanden-Eijnden},
    journal={Transactions on Machine Learning Research},
    issn={2835-8856},
    year={2025}
}

@inproceedings{pmlr-v202-pooladian23a,
    title = 	 {Multisample Flow Matching: Straightening Flows with Minibatch Couplings},
    author =       {Pooladian, Aram-Alexandre and Ben-Hamu, Heli and Domingo-Enrich, Carles and Amos, Brandon and Lipman, Yaron and Chen, Ricky T. Q.},
    booktitle = 	 {Proceedings of the 40th International Conference on Machine Learning},
    pages = 	 {28100--28127},
    year = 	 {2023},
    editor = 	 {Krause, Andreas and Brunskill, Emma and Cho, Kyunghyun and Engelhardt, Barbara and Sabato, Sivan and Scarlett, Jonathan},
    volume = 	 {202},
    series = 	 {Proceedings of Machine Learning Research},
    month = 	 {23--29 Jul},
    publisher =    {PMLR}
}

@inproceedings{
    liu2023flow,
    title={Flow Straight and Fast: Learning to Generate and Transfer Data with Rectified Flow},
    author={Xingchao Liu and Chengyue Gong and Qiang Liu},
    booktitle={The Eleventh International Conference on Learning Representations},
    year={2023}
}

@inproceedings{NEURIPS2023_c428adf7,
    author = {Shi, Yuyang and De Bortoli, Valentin and Campbell, Andrew and Doucet, Arnaud},
    booktitle = {Advances in Neural Information Processing Systems},
    editor = {A. Oh and T. Naumann and A. Globerson and K. Saenko and M. Hardt and S. Levine},
    pages = {62183--62223},
    publisher = {Curran Associates, Inc.},
    title = {Diffusion Schr\"{o}dinger Bridge Matching},
    volume = {36},
    year = {2023}
}

@inproceedings{
    bortoli2021diffusion,
    title={Diffusion Schr\"odinger Bridge with Applications to Score-Based Generative Modeling},
    author={Valentin De Bortoli and James Thornton and Jeremy Heng and Arnaud Doucet},
    booktitle={Advances in Neural Information Processing Systems},
    editor={A. Beygelzimer and Y. Dauphin and P. Liang and J. Wortman Vaughan},
    year={2021}
}

@article{plug_in_bridge,
    author = {Pooladian, Aram-Alexandre and Niles-Weed, Jonathan},
    title = {Plug-in Estimation of Schrödinger Bridges},
    journal = {SIAM Journal on Mathematics of Data Science},
    volume = {7},
    number = {3},
    pages = {1315-1336},
    year = {2025},
    doi = {10.1137/24M1687340},
    eprint = {https://doi.org/10.1137/24M1687340}
}

@inproceedings{coco,
    author="Lin, Tsung-Yi
    and Maire, Michael
    and Belongie, Serge
    and Hays, James
    and Perona, Pietro
    and Ramanan, Deva
    and Doll{\'a}r, Piotr
    and Zitnick, C. Lawrence",
    editor="Fleet, David
    and Pajdla, Tomas
    and Schiele, Bernt
    and Tuytelaars, Tinne",
    title="Microsoft COCO: Common Objects in Context",
    booktitle="Computer Vision -- ECCV 2014",
    year="2014",
    publisher="Springer International Publishing",
    address="Cham",
    pages="740--755"
}

@misc{coco_captions,
    title={Microsoft COCO Captions: Data Collection and Evaluation Server}, 
    author={Xinlei Chen and Hao Fang and Tsung-Yi Lin and Ramakrishna Vedantam and Saurabh Gupta and Piotr Dollar and C. Lawrence Zitnick},
    year={2015},
    eprint={1504.00325},
    archivePrefix={arXiv},
    primaryClass={cs.CV}
}

@inproceedings{fid,
    author = {Heusel, Martin and Ramsauer, Hubert and Unterthiner, Thomas and Nessler, Bernhard and Hochreiter, Sepp},
    booktitle = {Advances in Neural Information Processing Systems},
    editor = {I. Guyon and U. Von Luxburg and S. Bengio and H. Wallach and R. Fergus and S. Vishwanathan and R. Garnett},
    pages = {},
    publisher = {Curran Associates, Inc.},
    title = {GANs Trained by a Two Time-Scale Update Rule Converge to a Local Nash Equilibrium},
    volume = {30},
    year = {2017}
}

@inproceedings{is,
    author = {Salimans, Tim and Goodfellow, Ian and Zaremba, Wojciech and Cheung, Vicki and Radford, Alec and Chen, Xi and Chen, Xi},
    booktitle = {Advances in Neural Information Processing Systems},
    editor = {D. Lee and M. Sugiyama and U. Luxburg and I. Guyon and R. Garnett},
    pages = {},
    publisher = {Curran Associates, Inc.},
    title = {Improved Techniques for Training GANs},
    volume = {29},
    year = {2016}
}
\bibliographystyle{icml2026}

\newpage
\appendix
\onecolumn
\section{Proofs}\label{sec:app_proofs}%

\PropIntraConversion*
\begin{proof}
    For $t \in \interval[open]{0}{1}$ we can clearly isolate the score $s$ in each of the relations in the proposition. It therefore suffices to prove that each of the relations hold for all $(t, x) \in \interval[open]{0}{1} \times \reals^d$, as claimed.

    The first relation is a standard result that can be proved with Stein's lemma (also known as \emph{Gaussian integration by parts}) \cite{stein} saying that $s(t, x) = -\eta_Z (t, x) / \alpha_t$ for $(t, x) \in \interval[open right]{0}{1} \times \reals^d$, see e.g.~Theorem 8 in \citet{stochastic_interpolants}. The second relation follows easily from the definition of the stochastic interpolant (\cref{def:interpolant}) $I_t = \alpha_t Z + \beta_t X$ and linearity of expectation. Indeed, taking expectations conditional on $I_t = x$ in the definition for $I_t$ gives
    \[
        \eta_X (t, x) = \frac{1}{\beta_t} \ps{x - \alpha_t \eta_Z (t, x)} = \frac{1}{\beta_t} \ps{x + \alpha_t^2 s(t, x)},
    \]
    where we used the first relation to get the last equality.

    To get the final relation, plug the preceding relations into the definition of the drift
    \begin{align*}
        b^\eps (t, x) &= \dot{\alpha}_t \eta_Z (t, x) + \dot{\beta}_t \eta_X (t, x) + \eps_t s (t, x) \\
        &= \dot{\alpha}_t \ps{-\alpha_t s(t, x)} + \dot{\beta}_t \ps{\frac{1}{\beta_t} \ps{x + \alpha_t^2 s(t, x)}} + \eps_t s (t, x) \\
        &= \ps{\alpha_t^2 \frac{\dot{\beta}_t}{\beta_t} - \alpha_t \dot{\alpha}_t + \eps_t} s(t, x) + \frac{\dot{\beta}_t}{\beta_t} x.
    \end{align*}
    Since by \cref{def:eps_optimal} $\eps_t^* = \alpha_t^2 \dot{\beta}_t/\beta_t - \alpha_t \dot{\alpha}_t$ this concludes the proof of the proposition.
\end{proof}

\TheoremEpsOptKLDiv*
\begin{proof}
    An application of Girsanov's theorem gives
    \begin{align*}
        \kldiv(X^\eps, \hat{X}^\eps) &= \frac{1}{4} \int_0^1 \frac{1}{\eps_t} \ex\bs{\norm{b^\eps (t, X_t^\eps) - \hat{b}^\eps (t, X_t^\eps)}^2} \odif{t} \\
        &= \frac{1}{4} \int_0^1 \frac{1}{\eps_t} \ex\bs{\norm{b^\eps (t, I_t) - \hat{b}^\eps (t, I_t)}^2} \odif{t} \\
        &= \frac{1}{4} \int_0^1 \frac{(\eps_t^* + \eps_t)^2}{\eps_t} \ex\bs{\norm{s (t, I_t) - \hat{s} (t, I_t)}^2} \odif{t},
    \end{align*}
    where in the second equality we used that $\law{X_t^\eps} = \law{I_t}$ and in the third equality we used \cref{prop:intra_conversion} with the assumptions on $\hat{b}$ and $\hat{s}$ in this theorem. From the last expression we see that minimizing the KL-divergence over the diffusion scale $\eps$ amounts to minimizing the expression $(\eps_t^* + \eps_t)^2/\eps_t$ over $\eps_t$ for all $t \in \unitinterval$ since the expectation term does not depend on $\eps_t$. Differentiating this expression with respect to $\eps_t$ and setting the resulting expression to zero gives the (unique) minimizer $\eps_t = \eps_t^*$ for all $t \in \unitinterval$, and the theorem follows.
\end{proof}

\PropLinearEpsOptimal*
\begin{proof}
    We have $\overline{\alpha}_t = 1-t, \overline{\beta}_t = t$ so that $\dot{\overline{\alpha}}_t = -1$ and $\dot{\overline{\beta}}_t = 1$. Thus
    \[
        \overline{\eps}^*_t = \overline{\alpha}_t^2 \frac{\dot{\overline{\beta}}_t}{\overline{\beta}_t} - \overline{\alpha}_t \dot{\overline{\alpha}}_t = \frac{(1-t)^2}{t} + (1-t) = \frac{1-t}{t}.
    \]
    The rest of the proposition easily follows.
\end{proof}

Recall \cref{def:u_and_c}.
\DefUAndC*

For the remaining proofs we will often reference the following lemma with some useful algebraic identities.
\begin{lemma}
    \label{lemma:u_and_c}%
    For $t \in \interval[open left]{0}{1}$ the following identities hold.
    \begin{align}
        1 - u_t &= \frac{\alpha_t}{c_t}, \label{eq:1_minus_u} \\
        \frac{\dot{u}_t}{u_t} &= \frac{\dot{\beta}_t}{\beta_t} - \frac{\dot{c}_t}{c_t} = \frac{\eps^*_t}{\alpha_t c_t}. \label{eq:dot_u_over_u}
    \end{align}
\end{lemma}
\begin{proof}
    \cref{eq:1_minus_u} is trivial:
    \[
        1 - u_t = \frac{c_t - \beta_t}{c_t} = \frac{\alpha_t}{c_t}.
    \]
    To get the leftmost equality in \cref{eq:dot_u_over_u}, observe that
    \[
        \dot{u}_t = \frac{\dot{\beta}_t c_t - \beta_t \dot{c}_t}{c_t^2},
    \]
    so
    \[
        \frac{\dot{u}_t}{u_t} = \frac{\dot{\beta}_t c_t - \beta_t \dot{c}_t}{c_t \beta_t} = \frac{\dot{\beta}_t}{\beta_t} - \frac{\dot{c}_t}{c_t}.
    \]
    Finally,
    \[
        c_t \frac{\dot{u}_t}{u_t} = \alpha_t \frac{\dot{\beta}_t}{\beta_t} + \dot{\beta}_t - \dot{\alpha}_t - \dot{\beta}_t = \frac{\eps_t^*}{\alpha_t},
    \]
    which proves the rightmost equality in \cref{eq:dot_u_over_u}.
\end{proof}

\PropInterConversion*
\begin{proof}
    Using \cref{eq:1_minus_u} we get
    \begin{align*}
        \overline{\eta}_Z \ps{u_t, \frac{x}{c_t}} &= \ex\bs{Z \given (1 - u_t) Z + u_t X = \frac{x}{c_t}} \\
        &= \ex\bs{Z \given \frac{\alpha_t}{c_t} Z + \frac{\beta_t}{c_t} X = \frac{x}{c_t}} \\
        &= \ex\bs{Z \given \alpha_t Z + \beta_t X = x} \\
        &= \eta_Z (t, x).
    \end{align*}
    An analogous derivation proves $\eta_X (t, x) = \overline{\eta}_X \ps{u_t, x/c_t}$.

    From the relation $s(t, x) = -(1/\alpha_t) \eta_Z (t, x)$ (as in \cref{prop:intra_conversion}) combined with $\eta_Z (t, x) = \overline{\eta}_Z \ps{u_t, x/c_t}$ we get
    \[
        \overline{s}\ps{u_t, \frac{x}{c_t}} = -\frac{1}{1-u_t} \overline{\eta}_Z \ps{u_t, \frac{x}{c_t}} = -\frac{1}{1 - u_t} \eta_Z (t, x).
    \]
    Observing that $1/(1-u_t) = c_t/\alpha_t$ we get
    \[
        \frac{1}{c_t} \overline{s}\ps{u_t, \frac{x}{c_t}} = -\frac{1}{\alpha_t} \eta_Z (t, x) = s(t, x),
    \]
    as claimed.

    Getting the final relation requires a bit more work, but the overall strategy is the same as for the above; we first convert the drift to the score, then the score to the linear score, and then the linear score to the linear drift using the particular linear diffusion scale $\overline{\eps}_{u_t} = (\alpha_t \eps_t)/(\beta_t \eps^*_t)$. To be precise, we first apply \cref{prop:intra_conversion}, then the relation we just proved, and then \cref{prop:intra_conversion} again to get
    \begin{align*}
        b^\eps (t, x) &= (\eps^*_t + \eps_t) s(t, x) + \frac{\dot{\beta}_t}{\beta_t}x \\
        &= \frac{\eps^*_t + \eps_t}{c_t} \overline{s} (u_t, \frac{x}{c_t}) + \frac{\dot{\beta}_t}{\beta_t}x \\
        &= \frac{\eps^*_t + \eps_t}{c_t} \ps{\frac{\overline{b}^{\overline{\eps}} \ps{u_t, \frac{x}{c_t}} - \frac{x}{u_t c_t}}{\overline{\eps}_{u_t}^* + \overline{\eps}_{u_t}}} + \frac{\dot{\beta}_t}{\beta_t}x.
    \end{align*}
    We now plug in the linear diffusion scale $\overline{\eps}_{u_t} = (\alpha_t \eps_t)/(\beta_t \eps^*_t)$, which we motivate in the proof of \cref{theorem:path_transform_sde}. By \cref{prop:linear_eps_optimal} and \cref{eq:1_minus_u} we get $\overline{\eps}_{u_t}^* = (1-u_t)/u_t = \alpha_t/\beta_t$ so that
    \[
        \overline{\eps}_{u_t}^* + \overline{\eps}_{u_t} = \frac{\alpha_t}{\beta_t}\ps{\frac{\eps_t^* + \eps_t}{\eps_t^*}}.
    \]
    Thus
    \begin{align*}
         b^\eps (t, x) &= \frac{\beta_t \eps_t^*}{\alpha_t c_t} \ps{\overline{b}^{\overline{\eps}} \ps{u_t, \frac{x}{c_t}} - \frac{x}{u_t c_t}} + \frac{\dot{\beta}_t}{\beta_t}x \\
         &= c_t \dot{u}_t \ps{\overline{b}^{\overline{\eps}} \ps{u_t, \frac{x}{c_t}} - \frac{x}{u_t c_t}} + \frac{\dot{\beta}_t}{\beta_t}x \tag{using \cref{eq:dot_u_over_u}} \\
         &= c_t \dot{u}_t \overline{b}^{\overline{\eps}} \ps{u_t, \frac{x}{c_t}} + \ps{\frac{\dot{\beta}_t}{\beta_t} - \frac{\dot{u}_t}{u_t}}x \\
         &= c_t \dot{u}_t \overline{b}^{\overline{\eps}} \ps{u_t, \frac{x}{c_t}} + \frac{\dot{c}_t}{c_t}x \tag{using \cref{eq:dot_u_over_u}}.
    \end{align*}
    This concludes the proof.
\end{proof}

\TheoremPathTransformSDE*
\begin{proof}
    First of all we note that $\overline{W}$ is indeed a Wiener process in reparameterized time $u = u_t$ by Lévy's characterization theorem (see e.g.~Theorem 5.12 in \citet{LeGall2016Brownian}) since it has quadratic variation $\quadraticvar{\overline{W}}_u = \int_0^{t_u} \dot{u}_s \odif{s} = \int_0^u \odif{u} = u$. Thus both SDEs in the theorem statement are well-defined with unique solutions, so it suffices to prove that their dynamics coincide i.e.~that $\odif{X_t^\eps} = \odif{\ps{c_t \overline{X}_{u_t}^{\overline{\eps}}}}$ for all $t \in \interval[open]{0}{1}$. By the Itô product rule, the latter infinitesimal is
    \begin{align*}
        \odif{\ps{c_t \overline{X}_{u_t}^{\overline{\eps}}}} &= \dot{c}_t \overline{X}_{u_t}^{\overline{\eps}} \odif{t} + c_t \overline{b}^{\overline{\eps}} \ps{u_t, \overline{X}_{u_t}^{\overline{\eps}}} \odif{u_t} + c_t \sqrt{2 \overline{\eps}_{u_t}} \odif{\overline{W}_{u_t}} \\
        &= \dot{c}_t \overline{X}_{u_t}^{\overline{\eps}} \odif{t} + c_t \dot{u}_t \overline{b}^{\overline{\eps}} \ps{u_t, \overline{X}_{u_t}^{\overline{\eps}}} \odif{t} + c_t \sqrt{2 \overline{\eps}_{u_t} \dot{u}_t}\odif{W_t}.
    \end{align*}
    Using the linear diffusion scale $\overline{\eps}_{u_t} = (\alpha_t \eps_t)/(\beta_t \eps_t^*)$ with \cref{eq:dot_u_over_u} gives
    \[
        \overline{\eps}_{u_t} = \frac{u_t}{\dot{u}_t} \frac{\eps_t}{\beta_t c_t} = \frac{\eps_t}{c_t^2 \dot{u}_t},
    \]
    which we plug into the above infinitesimal to arrive at
    \[
        \odif{\ps{c_t \overline{X}_{u_t}^{\overline{\eps}}}} = \dot{c}_t \overline{X}_{u_t}^{\overline{\eps}} \odif{t} + c_t \dot{u}_t \overline{b}^{\overline{\eps}} \ps{u_t, \overline{X}_{u_t}^{\overline{\eps}}} \odif{t} + \sqrt{2 \eps_t} \odif{W_t}.
    \]
    Using now \cref{prop:inter_conversion} we see that
    \[
        \odif{X_t^\eps} = b^\eps (t, X_t^\eps) \odif{t} + \sqrt{2 \eps_t} \odif{W}_t = \frac{\dot{c}_t}{c_t}X^\eps_t \odif{t} + c_t \dot{u}_t \overline{b}^{\overline{\eps}} \ps{u_t, \frac{X^\eps_t}{c_t}} \odif{t} + \sqrt{2 \eps_t} \odif{W_t},
    \]
    which after inserting the theorem's ansatz $X_t^\eps/c_t = \overline{X}_{u_t}^{\overline{\eps}}$ shows that $\odif{X_t^\eps} =  \odif{\ps{c_t \overline{X}_{u_t}^{\overline{\eps}}}}$. This concludes the proof.
\end{proof}

To prove \cref{theorem:point_mass_generation} we first state and prove a useful lemma.
\begin{lemma}
    \label{lemma:u_0_cond}
    For any point mass schedule $(\alpha, \beta)$ it holds that
    \[
        u_0 \define \lim_{t \to 0_+} u_t = 0 \iff \beta_t = o(\alpha_t) \quad \text{as} \quad t \to 0_+.
    \]
\end{lemma}
\begin{proof}
    We have that
    \[
        u_0 \define \lim_{t \to 0_+} u_t = \lim_{t \to 0_+} \frac{\beta_t}{\alpha_t + \beta_t} = \lim_{t \to 0_+} \frac{1}{1 + \frac{\alpha_t}{\beta_t}}.
    \]
    Clearly this limit is zero if and only if the ratio $\alpha_t/\beta_t \to \infty$ as $t \to 0_+$, or equivalently if and only if $\beta_t = o(\alpha_t)$ as $t \to 0_+$.
\end{proof}

\TheoremPointMassGeneration*
\begin{proof}
    Since the point mass schedule from \cref{def:point_mass_schedule} only differs from the regular schedule from \cref{def:interpolant} at $t=0$ this is the only case we need to consider. In particular, \cref{theorem:path_transform_sde} holds on $\interval[open left]{0}{1}$ and also in the limit as $t \to 0_+$ since from the boundary conditions $\alpha_0 = \beta_0 = 0$ we have $c_0 = u_0 = 0$ so that $X_0^\eps = c_0 \overline{X}_0^{\overline{\eps}} = 0$ for all $(z, \omega) \in \reals^d \times \Omega$ with $z = \overline{X}_0^{\overline{\eps}}$ the initial condition to the linear interpolant and $\omega$ an event in the probability space of the shared Wiener process.
    
    Substituting now $c_t x$ for $x$ in \cref{prop:inter_conversion} gives the conversion formula
    \[
        b^\eps (t, c_t x) = \dot{c}_t x + c_t \dot{u}_t \overline{b}^{\overline{\eps}} \ps{u_t, x},
    \]
    for all $(t, x) \in \interval[open]{0}{1} \times \reals^d$ with $\overline{\eps}_{u_t} = \frac{\alpha_t \eps_t}{\beta_t \eps_t^*}$. Going to the limit gives us
    \begin{align*}
        \lim_{t \to 0_+} b^\eps(t, c_t x) = \lim_{t \to 0_+} \dot{c}_t x + \lim_{t \to 0_+} c_t \dot{u}_t \overline{b}^{\overline{\eps}} \ps{u_t, x},
    \end{align*}
    so it remains to prove that $\lim_{t \to 0_+} c_t \dot{u}_t \overline{b}^{\overline{\eps}} \ps{u_t, x} = - \lim_{t \to 0_+} (\eps_t / \alpha_t )x$ under our assumptions in \cref{def:point_mass_schedule}. To this end we first consider the expression for the drift of the linear interpolant,
    \begin{align*}
        \overline{b}^{\overline{\eps}} \ps{u_t, x} &= \ps{\overline{\eps}_{u_t}^* + \overline{\eps}_{u_t}} \overline{s} (u_t, x) + \frac{\dot{\overline{\beta}}_{u_t}}{\overline{\beta}_{u_t}} x \tag{by \cref{prop:intra_conversion}}\\
        &= \ps{\frac{1-u_t}{u_t} + \frac{\alpha_t \eps_t}{\beta_t \eps^*_t}} \ps{-\frac{\overline{\eta}_Z (u_t, x)}{1-u_t}} + \frac{1}{u_t} x \tag{by \cref{prop:intra_conversion} and \cref{prop:linear_eps_optimal}}\\
        &= -\frac{\alpha_t}{\beta_t} \ps{1 + \frac{\eps_t}{\eps_t^*}} \frac{\overline{\eta}_Z (u_t, x) c_t}  {\alpha_t} + \frac{c_t}{\beta_t} x \tag{by \cref{eq:1_minus_u}}\\
        &= \frac{c_t}{\beta_t} \ps{x - \overline{\eta}_Z (u_t, x) - \frac{\eps_t}{\eps_t^*}\overline{\eta}_Z (u_t, x)} \\
        &= \frac{c_t}{\beta_t} \ps{x - \overline{\eta}_Z (u_t, x)} - \frac{u_t \eps_t}{\dot{u}_t \alpha_t \beta_t} \overline{\eta}_Z (u_t, x) \tag{by \cref{eq:dot_u_over_u}}\\
        &= \frac{c_t}{\beta_t} \ps{x - \overline{\eta}_Z (u_t, x)} - \frac{\eps_t}{\dot{u}_t \alpha_t c_t} \overline{\eta}_Z (u_t, x).
    \end{align*}
    Multiplying by $c_t \dot{u}_t$ and going to the limit with the assumptions on the point mass schedule $(\alpha, \beta)$ from \cref{def:point_mass_schedule} finally gives
    \begin{align*}
        c_t \dot{u}_t \overline{b}^{\overline{\eps}} \ps{u_t, x} &= \frac{c_t^2 \dot{u}_t}{\beta_t} \ps{x - \overline{\eta}_Z (u_t, x)} - \frac{\eps_t}{\alpha_t} \overline{\eta}_Z (u_t, x)  \\
        &\to -\lim_{t \to 0_+} \frac{\eps_t}{\alpha_t} x \quad \text{as} \quad t \to 0_+,
    \end{align*}
    where to get the limit we used that $\overline{\eta}_Z (u_t, x) = \ex\bs{Z \given \overline{I}_{u_t} = x} \to x$ as $t \to 0_+$ since $u_0 = 0, \dot{u}_0 < \infty$ (and $\overline{I}_0 = Z$), as well as $c_t^2 \dot{u}_t / \beta_t = \dot{u}_t (\alpha_t^2 + \beta_t^2 + 2\alpha_t \beta_t) / \beta_t \to \alpha_t^2 \dot{u}_t/\beta_t \to C$ as $t \to 0_+$ for some non-negative bounded constant $C \in \reals_+$ since $\dot{u}_0 < \infty, \alpha_0 = \beta_0 = 0$, and since the point mass schedule is assumed to satisfy $\alpha_t^2 = O(\beta_t)$ as $t \to 0_+$. This concludes the proof.
    
    We note that the point mass schedule assumption $\alpha_t^2 = O(\beta_t)$ as $t \to 0_+$ can be relaxed, but this comes at the expense of requiring $\overline{\eta}_Z (u_t, x)$ to converge sufficiently fast to $x$ as $t \to 0_+$ so that the limiting product $\lim_{t \to 0_+} ((\alpha_t^2 \dot{u}_t)/\beta_t) (x - \overline{\eta}_Z(u_t, x))$ still goes to zero. To the best of our knowledge, this requires assuming that the data distribution $\rho_X$ has sufficiently light tails; see Assumption B.3 and Theorem B.4 in \citet{probabilistic_forecasting}. We feel like our assumption is simpler, and for what is done in this paper the assumption $\alpha_t^2 = O(\beta_t)$ as $t \to 0_+$ is not prohibitively restrictive.
\end{proof}

\PropKLDivInvariance*
\begin{proof}
    This is essentially already proved in Proposition 2.5 in \citet{lipschitz_schedules} whose proof we restate below using the notation of this paper. Recall that under the assumptions in \cref{theorem:eps_opt_kl_div}, the minimum Kullback-Leibler divergence between the path measures of $X^\eps$ and $\hat{X}^\eps$ over the diffusion scale $\eps$ is
    \[
        \kldiv\ps{X^{\eps^*}, \hat{X}^{\eps^*}} = \int_0^1 \eps^*_t \ex\bs{\norm{s (t, I_t) - \hat{s} (t, I_t)}^2} \odif{t}.
    \]
    Define now the time change $\eta: \unitinterval \to \extreals_+$ by $\eta_t \define \alpha_t/\beta_t$, let $\rho_\eta$ be the probability density function of the random variable $Y = X + \eta Z$ for $\eta \in \interval{0}{\infty}$ and denote by $s (\eta, y)$ the score of $Y$\footnote{We are abusing notation here, but in this proof it will always be clear from the argument ($t$ or $\eta$) whether we are talking about the density of $I$ or $Y$.}. Since from \cref{def:interpolant} we can write $I_t = \alpha_t Z + \beta_t X = \beta_t (X + \eta_t Z) = \beta_t Y$ we see that the scores are related as
    \[
        s(t, x) = \grad \log \rho_t (x) = \frac{1}{\beta_t} \grad \log \rho_\eta \ps{\frac{x}{\beta_t}} = \frac{1}{\beta_t} s \ps{\eta, \frac{x}{\beta_t}},
    \]
    and similarly we have that the score estimates are related as $\hat{s} (t, x) = \hat{s}(\eta, x/\beta)/\beta_t$. Substituting the score of $Y$ for the score of $I$ in the expression for the KL-minimum gives
    \begin{align*}
        \kldiv\ps{X^{\eps^*}, \hat{X}^{\eps^*}} &= \int_0^1 \frac{\eps^*_t}{\beta_t^2} \ex\bs{{\norm{s \ps{\eta_t, \frac{I_t}{\beta_t}} - \hat{s} \ps{\eta_t, \frac{I_t}{\beta_t}}}}^2} \odif{t} \\
        &= \int_0^1 \frac{\eps^*_t}{\beta_t^2} \ex\bs{{\norm{s \ps{\eta_t, Y} - \hat{s} \ps{\eta_t, Y}}}^2} \odif{t}. \\
    \end{align*}
    Noting that
    \[
        \frac{\eps_t^*}{\beta_t^2} = \frac{\alpha_t^2}{\beta_t^2} \ps{\frac{\dot{\beta}_t}{\beta_t} - \frac{\dot{\alpha}_t}{\alpha_t}} = - \frac{\alpha_t}{\beta_t} \odv{}{t} \ps{\frac{\alpha_t}{\beta_t}} = -\eta_t \dot{\eta}_t,
    \]
    we change the integration variable from $t$ to $\eta$ in the KL-expression to get
    \[
        \kldiv\ps{X^{\eps^*}, \hat{X}^{\eps^*}} = \int_{\eta_1}^{\eta_0} \eta \ex\bs{{\norm{s \ps{\eta, Y} - \hat{s} \ps{\eta, Y}}}^2} \odif{\eta} = \int_0^\infty \eta \ex\bs{{\norm{s \ps{\eta, Y} - \hat{s} \ps{\eta, Y}}}^2}\odif{\eta}.
    \]
    Crucially, we substituted $\eta_1 = 0$ but also $\eta_0 \define \lim_{t \to 0_+} \eta_t = \infty$ in the last expression, as in the original proof. We stress that $\eta_0 = \infty$ holds even for a point mass schedule since in \cref{def:point_mass_schedule} we require $\beta_t = o(\alpha_t)$ as $t \to 0_+$, which is equivalent to requiring $u_0 = 0$ by \cref{lemma:u_0_cond}. This requirement cannot be relaxed without making the initial law of the resulting interpolant dependent on $\rho_X$, which is what we are trying to sample in the first place. Since the final expression for the KL-minimum is independent of the schedule $(\alpha, \beta)$, the proof is concluded.
\end{proof}

\PropLazyScheduleFamilies*
\begin{proof}
    Under the assumption $X \sample \standardnormal$ it holds that $I_t \sample \normal(0, (\alpha_t^2 + \beta_t^2)\id)$ so that $s(t, x) = -x/(\alpha_t^2 + \beta_t^2)$. We first consider what this entails for the ODE case and then for the SDE case. It suffices to consider the condition $b(t, x) = 0$ for all $(t, x) \in \interval[open]{0}{1} \times \reals^d$ since by continuity the condition $b(t, x) = 0$ for all $(t, x) \in \interval{0}{1} \times \reals^d$ is then covered too.

    From \cref{prop:intra_conversion} with $\eps \equiv 0$ we see that requiring $b(t, x) = 0$ for all $(t, x) \in \interval[open]{0}{1} \times \reals^d$ is equivalent to requiring $\eps_t^* s(t, x) + \dot{\beta}_t/\beta_t x = 0$ for all $(t, x) \in \interval[open]{0}{1} \times \reals^d$. This is equivalent to
    \[
        \frac{\dot{\beta}_t}{\beta_t} - \frac{\eps^*_t}{\alpha_t^2 + \beta_t^2} = 0 \quad \forall t \in \interval[open]{0}{1},
    \]
    which by multiplying by $\alpha_t^2 + \beta_t^2$ is equivalent to
    \[
        \eps_t^* = \alpha_t^2 \frac{\dot{\beta}_t}{\beta_t} + \beta_t \dot{\beta}_t \quad \forall t \in \interval[open]{0}{1},
    \]
    which from \cref{def:eps_optimal} is again equivalent to
    \[
        \alpha_t \dot{\alpha}_t = -\beta_t \dot{\beta}_t \quad \forall t \in \interval[open]{0}{1}.
    \]
    This is also equivalent to
    \[
        \odv{}{t} \ps{\alpha_t^2 + \beta_t^2} = C \quad \forall t \in \interval[open]{0}{1},
    \]
    for some constant $C$. The boundary conditions $\alpha_1 = \beta_0 = 0, \alpha_0 = \beta_1 = 1$ then give $\alpha_t^2 + \beta_t^2 = 1$, proving the ODE case.

    Next, we consider the SDE case. Here we use \cref{prop:intra_conversion} with $\eps = \eps^*$ to see that requiring $b^*(t, x) = 0$ for all $(t, x) \in \interval[open]{0}{1} \times \reals^d$ is equivalent to requiring $2\eps_t^* s(t, x) + \dot{\beta}_t/\beta_t x = 0$ for all $(t, x) \in \interval[open]{0}{1} \times \reals^d$. This is equivalent to
    \[
        \frac{\dot{\beta}_t}{\beta_t} - \frac{2\eps^*_t}{\alpha_t^2 + \beta_t^2} = 0 \quad \forall t \in \interval[open]{0}{1},
    \]
    which by multiplying by $\alpha_t^2 + \beta_t^2$ is equivalent to
    \[
        2\eps_t^* = \alpha_t^2 \frac{\dot{\beta}_t}{\beta_t} + \beta_t \dot{\beta}_t \quad \forall t \in \interval[open]{0}{1},
    \]
    which from \cref{def:eps_optimal} is again equivalent to
    \[
        \alpha_t^2 \frac{\dot{\beta}_t}{\beta_t} - 2\alpha_t \dot{\alpha}_t = \beta_t \dot{\beta}_t \quad \forall t \in \interval[open]{0}{1}.
    \]
    Denoting $a(t) \define \alpha_t^2$ and observing that $\dot{a}(t) = 2\alpha_t \dot{\alpha}_t$ this is a first-order linear inhomogeneous ODE of the form
    \[
        \dot{a}(t) - \frac{\dot{\beta}_t}{\beta_t} a(t) = -\beta_t \dot{\beta}_t.
    \]
    Its solutions are of the form $a(t) = C \beta_t - \beta_t^2$. Imposing the boundary conditions $\alpha_1 = 0, \beta_1 = 1$ gives $C = 1$ so that $\alpha_t = \sqrt{\beta_t (1 - \beta_t)}$, or equivalently $\alpha_t^2 + \beta_t^2 = \beta_t$, as claimed.

    Note that $\beta_t/\alpha_t = \sqrt{\beta_t}/\sqrt{1-\beta_t} \to 0$ as $t \to 0_+$, or equivalently $\beta_t = o(\alpha_t)$ as $t \to 0_+$, as claimed. Similarly, $\alpha_t^2/\beta_t = 1-\beta_t \to 1$ as $t \to 0_+$ so that $\alpha_t^2 = O(\beta_t)$ as $t \to 0_+$. Differentiating yields
    \[
        \dot{u}_t = \odv{}{t} \ps{\frac{\beta_t}{\sqrt{\beta_t(1-\beta_t)} + \beta_t}} = \frac{\dot{\beta}_t}{2\sqrt{\beta_t(1-\beta_t)}\ps{1+2\sqrt{\beta_t(1-\beta_t)}}} \to \lim_{t \to 0_+} \frac{\dot{\beta}_t}{2\sqrt{\beta_t}} \quad \text{as} \quad t \to 0_+,
    \]
    which proves that $\dot{u}_0 < \infty$ if and only if $\dot{\beta}_t = O(\sqrt{\beta}_t)$ as $t \to 0_+$, as claimed. In that case we have shown that the schedule is indeed a point mass schedule as per \cref{def:point_mass_schedule}.
    
    Finally, we prove the claimed identity for $\eps^*_t$. We observe that $\dot{\alpha}_t = \dot{\beta}_t (1 - 2\beta_t)/(2\alpha_t)$ so that
    \[
        \eps^*_t = \alpha_t^2 \frac{\dot{\beta}_t}{\beta_t} - \alpha_t \dot{\alpha}_t = \dot{\beta}_t (1 - \beta_t) - \frac{1}{2} \dot{\beta}_t (1 - 2\beta_t) = \frac{1}{2} \dot{\beta}_t,
    \]
    as claimed. It follows that also $\int_s^t 2 \eps^*_u \odif{u} = \beta_t - \beta_s$ for any $0 \leq s \leq t \leq 1$.
\end{proof}

\PropODEOptScheduleTransform*
\begin{proof}
    Observing that $c_t = 1/\sqrt{d_t}, \dot{c}_t = (1 - 2t)/(d_t \sqrt{d_t}), u_t = t, \dot{u}_t = 1$ and applying \cref{prop:inter_conversion} with $\eps \equiv 0$ so that also $\overline{\eps} \equiv 0$ gives the conversion formula
    \[
        b (t, x) = \frac{1 - 2t}{d_t} x + \frac{1}{\sqrt{d_t}} \overline{b} (t, \sqrt{d_t}x).
    \]
    In particular, since $d_0 = 1$ and $u_0 = 0$ we get
    \[
        b(0, z) = z + \overline{b} (0, z) = z + \ex\bs{X \given Z = z} - \ex\bs{Z \given Z = z} = \ex\bs{X}.
    \]
    Note that all equations are well-defined even when $t \in \set{0, 1}$. This concludes the proof.
\end{proof}

\PropSDEOptScheduleTransform*
\begin{proof}
    Observing that $c_t = t/d_t, \dot{c}_t = (1 - 2t^2)/d_t^2, u_t = t, \dot{u}_t = 1$ and applying \cref{prop:inter_conversion} with $\eps = \eps^*$ so that
    \[
        \overline{\eps}_{u_t} = \overline{\eps}_t = \frac{\alpha_t}{\beta_t} = \frac{1-t}{t} = \overline{\eps}^*_t
    \]
    gives the conversion formula
    \begin{align*}
        b^* (t, x) = \frac{1-2t^2}{t d_t} x + \frac{t}{d_t} \overline{b}^* \ps{t, \frac{d_t}{t}x},
    \end{align*}
    which is valid for $t \in \interval[open left]{0}{1}$. With \cref{prop:intra_conversion} we can convert the statistically optimal SDE drift $\eps = \eps^*$ to the ODE velocity ($\eps \equiv 0$) , which when applied to the linear interpolant evaluated in $(t, (d_t/t)x)$ reads
    \[
        \overline{b}^* \ps{t, \frac{d_t}{t}x} = 2 \overline{b} \ps{t, \frac{d_t}{t}x} - \frac{\dot{\overline{\beta}}_t}{\overline{\beta}_t} \frac{d_t}{t}x = 2 \overline{b} \ps{t, \frac{d_t}{t}x} - \frac{d_t}{t^2}x.
    \]
    Inserting this into the expression for $b^* (t, x)$ gives
    \[
        b^* (t, x) = \frac{1-2t^2 - d_t}{t d_t}x + \frac{2t}{d_t} \overline{b} \ps{t, \frac{d_t}{t}x} = \frac{2}{d_t}\ps{(1-2t)x + t \overline{b} \ps{t, \frac{d_t}{t}x}},
    \]
    which is still only valid for $t \in \interval[open left]{0}{1}$. This proves the first part of the proposition. To analyze the case $t=0$ we substitute $(t/d_t)x$ for $x$ to get
    \[
        b^* \ps{t, \frac{t}{d_t}x} = \frac{2t}{d_t^2}(1-2t)x + \frac{2t}{d_t} \overline{b} (t, x) \to 0 \quad \text{as} \quad t \to 0_+,
    \]
    since $d_0 = 1$ and $\overline{b} (0, x) = \ex\bs{X} - x < \infty$ for all $x \in \reals^d$. As also $(t/d_t)x \to 0$ as $t \to 0_+$ we see that the initial drift satisfies $b^* (0, 0) = 0$, as claimed. This concludes the proof.
\end{proof}

\pagebreak
\section{Relationship to flows and diffusions}
In this section we rigorously describe how the stochastic interpolant from \cref{def:interpolant} relates to the objects studied in flow and diffusion models, respectively. The relationship to diffusion models is analyzed in section 5.1 in \cite{stochastic_interpolants} but to the best of our knowledge the particular connection to the statistically optimal diffusion scale $\eps^*$ (see \cref{def:eps_optimal} and \cref{theorem:eps_opt_kl_div}) as presented in \cref{prop:eps_opt_gives_OU_rev_process} has not been pointed out before.

First we prove some fundamental properties of the optimal diffusion scale $\eps^*$.
\begin{proposition}
    \label{prop:eps_opt}
    For any schedule $(\alpha, \beta)$ satisfying the criteria in \cref{def:interpolant}, $\eps^*$ from \cref{def:eps_optimal} satisfies
    \begin{align*}
    \eps_0^* &\define \lim_{t\rightarrow 0_+} \eps^*_t = \infty \quad \text{if } \lim_{t \to 0_+} \frac{\dot{\beta}_t}{\beta_t} \text{ exists}, \\
    \eps^*_t &\in \interval[open]{0}{\infty} \quad \forall t \in \interval[open]{0}{1}, \\
    \eps^*_1 &= 0, \\
    \int_0^t \eps^*_s \odif{s} &= \infty \quad \forall t \in \interval[open left]{0}{1}.
    \end{align*}
\end{proposition}
\begin{proof}
    We have
    \[
        \eps_0^* \define \lim_{t \to 0_+} \eps^*_t = \lim_{t \to 0_+} \ps{\alpha_t^2 \frac{\dot{\beta}_t}{\beta_t} - \alpha_t \dot{\alpha}_t} = \lim_{t \to 0_+} \frac{\dot{\beta}_t}{\beta_t} - \dot{\alpha}_0 = \lim_{t \to 0_+} \odv{}{t} \log(\beta_t) - \dot{\alpha}_0 = \infty,
    \]
    if $\lim_{t \to 0_+} (\dot{\beta}_t/\beta_t)$ exists since $\alpha_0 = 1, \abs{\dot{\alpha}_0} < \infty, \beta_0 = 0$ as well as $\beta_t > 0$ for $t > 0$. Strict positivity of $\eps^*$ for $t \in \interval[open]{0}{1}$ follows from the monotonicity assumptions $\dot{\alpha}_t < 0$ and $\dot{\beta}_t > 0$ for all $t \in \interval[open]{0}{1}$. $\eps^*_1 = 0$ follows from $\alpha_1 = 0, \beta_1 = 1, \abs*{\dot{\alpha}_1} < \infty, \abs*{\dot{\beta}_1} < \infty$.

    To see that $\eps^*$ is not integrable, we use $\dot{\alpha}_t < 0$ for $t \in \interval[open]{0}{1}$ to bound $\eps_t^* \geq \alpha_t^2 \odv{}{t} \log(\beta_t)$ and write
    \[
        \int_0^t \eps_u^* \odif{u} = \lim_{s \to 0_+} \int_s^t \eps_u^* \odif{u} \geq \alpha_t^2 \lim_{s \to 0_+} \int_s^t \odv{}{u} \log(\beta_u) \odif{u} = \alpha_t^2 (\log(\beta_t) - \lim_{s \to 0_+}\log(\beta_s)) = \infty \quad \forall t \in \interval[open left]{0}{1},
    \]
    since $\beta_0 = 0$ and $\beta \in C^1(\unitinterval)$. This concludes the proof.
\end{proof}

\subsection{Flow models} \label{sec:relation_to_flow_models}
In a flow model (or more precisely a \emph{flow matching model}) one considers the time-dependent random variable $(1-t)Z + t X$ for $Z \sample \standardnormal$ and $X \sample \rho_X$ independent with $t \in \unitinterval$. One then utilizes that the ODE
$$
    \odv{}{t} Y^\text{flow}_t = v^\text{flow}(t, Y^\text{flow}_t)
$$
with velocity field
$$
    v^\text{flow} (t, x) \define \ex\bs{X \given (1-t)Z + t X} - \ex\bs{Z \given (1-t)Z + t X}
$$
for $(t, x) \in \unitinterval \times \reals^d$ generates $\rho_X$, i.e.~$Y^\text{flow}_1 \sample \rho_X$ whenever $Y^\text{flow}_0 \sample \standardnormal$. One sees that this corresponds exactly to the linear interpolant from \cref{def:linear_interpolant}, i.e.~$\overline{I}_t = (1-t)Z + t X$ and
$$
    \overline{b} (t, x) = \ex\bs{\dot{\overline{I}}_t \given \overline{I}_t = x} = v^\text{flow} (t, x)
$$
for all $(t, x) \in \unitinterval \times \reals^d$, where we used the notation $\overline{b} \define \overline{b}^{\overline{\eps} \equiv 0}$ for the ODE velocity of the linear interpolant from \cref{def:linear_interpolant}.

\subsection{Diffusion models}
In a variance-preserving diffusion model (or more precisely a \emph{score-based generative model}) one corrupts data $\tilde{Y}_0 = X \sample \rho_X$ into noise $\tilde{Y}_\infty = Z \sample \standardnormal$ through an Ornstein-Uhlenbeck process run from time $0$ to $\infty$ (or some finite but large $T$ in practice which introduces a bias),
\[
  \odif{\tilde{Y}_s} = -f(s)\tilde{Y}_s \odif{s} + g(s)\odif{\tilde{W}_s},
\]
with smooth and non-negative $f,g$ satisfying
\[
  \int_0^\infty f(s)\odif{s} = \int_0^\infty g(s)\odif{s} = \infty,
\]
so that the stationary distribution is indeed standard Gaussian, i.e.~$\tilde{Y}_\infty \sample \standardnormal$.

One then uses Anderson's time reversal \cite{anderson_time_reversal} to construct the unique reverse-time process run from $\infty$ (or $T$) to $0$ that turns noise into data with the same joint distribution as the forward process. This reverse-time process depends on the score function of the Ornstein-Uhlenbeck process with initial distribution $\rho_X$.

To compare diffusion models to flow models and stochastic interpolants one must compactify time via some $C^1$ monotonically increasing bijection $\phi : \interval[open right]{0}{\infty} \to \interval[open right]{0}{1}$ satisfying $\phi(0)=0$ and $\lim_{s\to\infty}\phi(s)=1$, with inverse $\psi \define \phi^{-1} : \interval[open right]{0}{1} \to \interval[open right]{0}{\infty}$. One can then define $Y_\tau \define \tilde{Y}_{\psi(\tau)}$ for $\tau \in \interval[open right]{0}{1}$, and by a limiting argument extend to $\tau\in \unitinterval$. Using similar standard stochastic calculus techniques as in the proof of \cref{theorem:path_transform_sde} one gets that
\[
  \odif{Y_\tau}
  = -f(\psi(\tau))\psi'(\tau)\,Y_\tau \odif{\tau}
    + g(\psi(\tau))\sqrt{\psi'(\tau)} \odif{W_\tau},
\]
with
\[
  W_\tau \define \int_0^{\psi(\tau)} \sqrt{\phi'(s)} \odif{\tilde{W}_s}.
\]

In particular, if we take $\phi(s)=1-\exp(-s)$ so that $\psi(\tau)=-\log(1-\tau)$ and $\psi'(\tau)=1/(1-\tau)$, we get the Ornstein-Uhlenbeck process
\[
  \odif{Y_\tau}
  = -\frac{f(-\log(1-\tau))}{1-\tau}\,Y_\tau \odif{\tau}
    + \frac{g(-\log(1-\tau))}{\sqrt{1-\tau}}\odif{W_\tau}.
\]

As the following \cref{prop:eps_opt_gives_OU_rev_process} shows, when one chooses the statistically optimal diffusion scale $\eps=\eps^*$ in the stochastic interpolant framework, the corresponding reverse-time SDE becomes an Ornstein-Uhlenbeck process in the time range $\unitinterval$. Before stating this result we need to introduce the \emph{backward drift} as in \cite{stochastic_interpolants} as well as the core result analogous to \cref{theorem:sde} along with an intra-interpolant conversion result similar to \cref{prop:intra_conversion}.

\begin{definition}[backward drift]
    \label{def:backward_drift}
    For an interpolant as in \cref{def:interpolant} and $\eps: \unitinterval \rightarrow \extposreals, \eps \in C^1(\interval[open left]{0}{1})$ we define the backward drift
    \[
        \overset{\leftarrow}{b^\eps} (t, x) \define \dot{\alpha}_t \eta_Z (t, x) + \dot{\beta}_t \eta_X (t, x) - \eps_t s(t, x).
    \]
\end{definition}

We refer the reader to \citet{stochastic_interpolants} for a proof of the following result.
\begin{theorem}
    \label{theorem:backward_sde}
    Define $\tau \define 1 - t$ and $\overset{\leftarrow}{W_\tau} \define -W_\tau = -W_{1-t}$. Then the solutions to the family of reverse-time SDEs
    \[
        \odif{\overset{\leftarrow}{X^\eps_\tau}} = - \overset{\leftarrow}{b^\eps} (\tau, \overset{\leftarrow}{X^\eps_\tau}) \odif{\tau} + \sqrt{2 \eps_\tau} \odif{\overset{\leftarrow}{W_\tau}},
    \]
    solved forward in reverse-time $\tau$ with $\overset{\leftarrow}{X^\eps_0} \sample \rho_X$ independent of $\overset{\leftarrow}{W}$, satisfy $\law(\overset{\leftarrow}{X^\eps_\tau}) = \law(I_{1-\tau})$ for all $\tau \in \unitinterval$.
\end{theorem}

The following proposition can be proved analogously to \cref{prop:intra_conversion}.
\begin{proposition}
    \label{prop:intra_conversion_backward_drift}
    For a fixed interpolation schedule $(\alpha, \beta)$ and $\eps: \unitinterval \rightarrow \posreals$, the backward drift from \cref{def:backward_drift} can be expressed in terms of the score through the equation
    \[
        \overset{\leftarrow}{b^\eps} (t, x) = \ps{\eps_t^* - \eps_t} s(t, x) + \frac{\dot{\beta}_t}{\beta_t} x,
    \]
    for all $(t, x) \in \interval[open]{0}{1} \times \reals^d$.
\end{proposition}

We now state and prove our own result.
\begin{proposition}
    \label{prop:eps_opt_gives_OU_rev_process}
    Let $\overset{\leftarrow}{b^*} \define \overset{\leftarrow}{b^{\eps^*}}$ be the statistically optimal (in forward-time) reverse-time drift for any (non point mass) stochastic interpolant as per \cref{def:interpolant}. Then the associated reverse-time SDE from \cref{theorem:backward_sde} is a time-scaled Ornstein-Uhlenbeck process of the form
    \[
      \odif{\overset{\leftarrow}{X^*_\tau}}
      = -\frac{\dot{\beta}_\tau}{\beta_\tau} \overset{\leftarrow}{X^*_\tau} \odif{\tau}
        + \sqrt{2\eps^*_\tau} \odif{\overset{\leftarrow}{W_\tau}},
    \]
    solved forward in reverse time $\tau$ with initial condition $\overset{\leftarrow}{X^*_0} \sample \rho_X$.
    If additionally the interpolant is variance-preserving, i.e.~$\alpha_t^2+\beta_t^2=1$ for all $t \in \unitinterval$, then the reverse-time SDE reduces to
    \[
      \odif{\overset{\leftarrow}{X^*_\tau}}
      = -\eps^*_\tau \overset{\leftarrow}{X^*_\tau} \odif{\tau}
        + \sqrt{2\eps^*_\tau} \odif{\overset{\leftarrow}{W_\tau}}.
    \]
\end{proposition}

\begin{proof}
    The SDE from \cref{theorem:backward_sde} reads
    \[
      \odif{\overset{\leftarrow}{X^\eps_\tau}}
      = - \overset{\leftarrow}{b^\eps}(\tau, \overset{\leftarrow}{X^\eps_\tau}) \odif{\tau}
        + \sqrt{2\eps_\tau} \odif{\overset{\leftarrow}{W_\tau}}.
    \]
    Using \cref{prop:intra_conversion_backward_drift} with $\eps=\eps^*$ gives
    \[
      \overset{\leftarrow}{b^*}(\tau,x) = \frac{\dot{\beta}_\tau}{\beta_\tau} x,
    \]
    so that the SDE from \cref{theorem:backward_sde} reduces to
    \[
      \odif{\overset{\leftarrow}{X^*_\tau}}
      = -\frac{\dot{\beta}_\tau}{\beta_\tau} \overset{\leftarrow}{X^*_\tau} \odif{\tau}
        + \sqrt{2\eps_\tau^*} \odif{\overset{\leftarrow}{W_\tau}},
    \]
    as claimed. Assuming now that $\alpha_\tau^2 + \beta_\tau^2 = 1$, i.e.~assuming that the schedule is variance-preserving, gives
    \[
      \eps^*_\tau
      = (1-\beta_\tau^2)\frac{\dot{\beta}_\tau}{\beta_\tau}
        - \alpha_\tau\left(\frac{-\dot{\beta}_\tau \beta_\tau}{\alpha_\tau}\right)
      = \frac{\dot{\beta}_\tau}{\beta_\tau},
    \]
    so that the SDE from \cref{theorem:backward_sde} reduces to
    \[
      \odif{\overset{\leftarrow}{X^*_\tau}}
      = -\eps^*_\tau \overset{\leftarrow}{X^*_\tau} \odif{\tau}
        + \sqrt{2\eps^*_\tau}\odif{\overset{\leftarrow}{W_\tau}},
    \]
    as claimed.
\end{proof}
\pagebreak
\section{Schedule visualizations}

The first row of \cref{fig:schedules} visualizes the linear schedule (see \cref{def:linear_interpolant}) used in flow matching models alongside the lazy ODE and SDE schedule, respectively (see \cref{prop:lazy_schedule_families} and in particular \cref{ex:u_t_eq_t}). The second row visualizes the statistically optimal diffusion scale $\eps^*$ (see \cref{def:eps_optimal} and \cref{theorem:eps_opt_kl_div}) associated with each schedule. The figure uses the denominator definition $d_t \define (1-t)^2 + t^2$. Note that these are the schedules and diffusion scales used in \cref{fig:typst_interpolant_paths}.

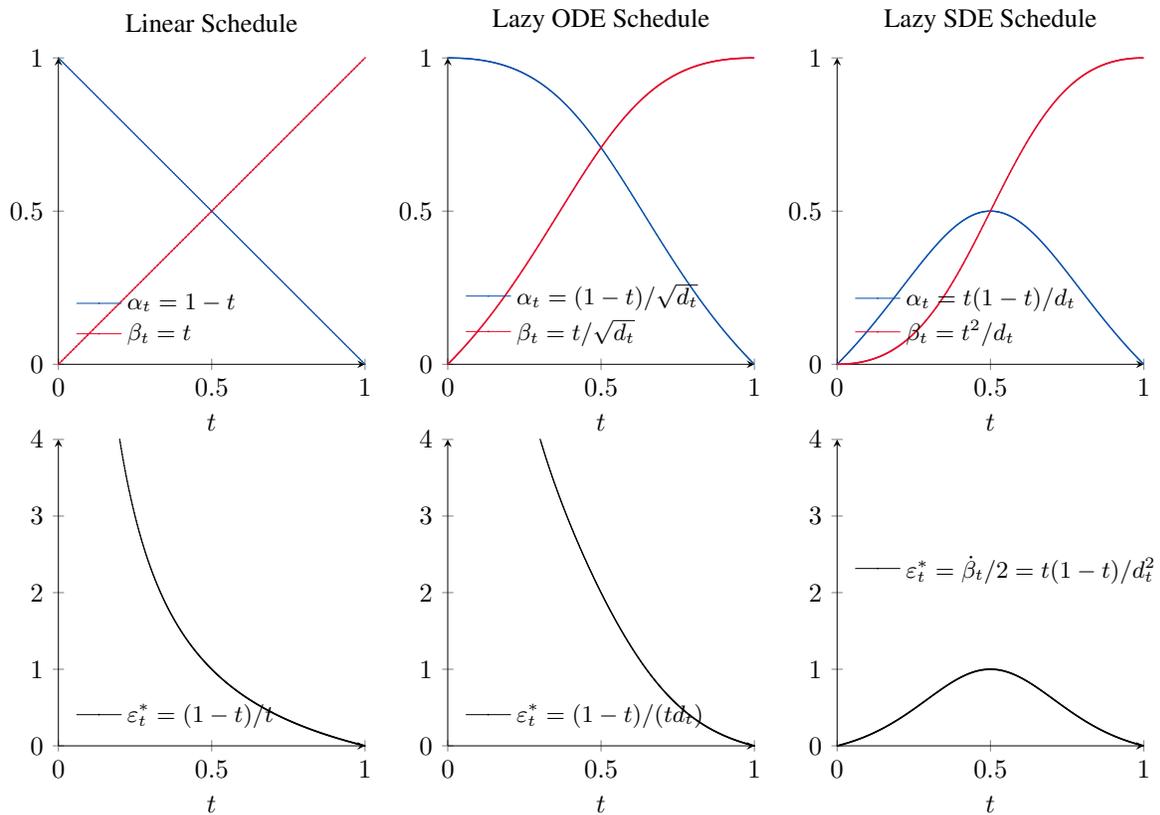
\begin{figure}[ht]
\centering
\begin{tikzpicture}

\pgfplotsset{
  every axis plot/.style={line width=0.30pt},
  every axis/.style={axis line style={line width=0.30pt}, tick style={line width=0.30pt}},
  sched/.style={
    width=0.33\textwidth,
    height=0.33\textwidth,
    xmin=0, xmax=1,
    ymin=0, ymax=1,
    xtick={0,0.5,1},
    ytick={0,0.5,1},
    xtick={0,0.5,1},
    ytick={0,0.5,1},
    xlabel={$t$},
    axis lines=left,
    legend style={
      draw=none,
      fill=none,
      font=\footnotesize,
      cells={anchor=west},
    },
    legend cell align=left,
  },
  epsplot/.style={
    width=0.33\textwidth,
    height=0.33\textwidth,
    xmin=0, xmax=1,
    ymin=0, ymax=4,
    xtick={0,0.5,1},
    ytick={0,1,2,3,4},
    xlabel={$t$},
    axis lines=left,
    legend style={
      draw=none,
      fill=none,
      font=\footnotesize,
      cells={anchor=west},
    },
    legend cell align=left,
  },
}

\definecolor{AlphaColor}{HTML}{064EA8}
\definecolor{BetaColor}{HTML}{E30C2C}

\begin{groupplot}[
  group style={
    group size=3 by 2,
    horizontal sep=1.1cm,
    vertical sep=1.0cm,
  },
]

\nextgroupplot[
  sched,
  title={Linear Schedule},
  legend style={at={(0.03,0.03)}, anchor=south west, draw=none, fill=none, font=\footnotesize, cells={anchor=west}},
]
\addplot[color=AlphaColor, mark=*, mark size=0.1pt, mark options={solid},
         domain=0:1, samples=200]
  {1 - x};
\addlegendentry{$\alpha_t=1-t$}
\addplot[color=BetaColor, mark=square*, mark size=0.1pt, mark options={solid},
         domain=0:1, samples=200]
  {x};
\addlegendentry{$\beta_t=t$}

\nextgroupplot[
  sched,
  title={Lazy ODE Schedule},
  legend style={at={(0.03,0.03)}, anchor=south west, draw=none, fill=none, font=\footnotesize, cells={anchor=west}},
]
\addplot[color=AlphaColor, mark=*, mark size=0.1pt, mark options={solid},
         domain=0:1, samples=400]
  {(1 - x)/sqrt((1 - x)^2 + x^2)};
\addlegendentry{$\alpha_t=(1-t)/\sqrt{d_t}$}
\addplot[color=BetaColor, mark=square*, mark size=0.1pt, mark options={solid},
         domain=0:1, samples=400]
  {x/sqrt((1 - x)^2 + x^2)};
\addlegendentry{$\beta_t=t/\sqrt{d_t}$}

\nextgroupplot[
  sched,
  title={Lazy SDE Schedule},
  legend style={at={(0.03,0.03)}, anchor=south west, draw=none, fill=none, font=\footnotesize, cells={anchor=west}},
]
\addplot[color=AlphaColor, mark=*, mark size=0.1pt, mark options={solid},
         domain=0:1, samples=400]
  {(x*(1 - x))/((1 - x)^2 + x^2)};
\addlegendentry{$\alpha_t=t(1-t)/d_t$}
\addplot[color=BetaColor, mark=square*, mark size=0.1pt, mark options={solid},
         domain=0:1, samples=400]
  {(x^2)/((1 - x)^2 + x^2)};
\addlegendentry{$\beta_t=t^2/d_t$}

\nextgroupplot[
  epsplot,
  legend style={at={(0.03,0.03)}, anchor=south west, draw=none, fill=none, font=\footnotesize, cells={anchor=west}},
]
\addplot[color=black, mark=diamond*, mark size=0.1pt, mark options={solid},
         domain=0.02:1, samples=600]
  {(1 - x)/x};
\addlegendentry{$\eps_t^*=(1-t)/t$}

\nextgroupplot[
  epsplot,
  legend style={at={(0.03,0.03)}, anchor=south west, draw=none, fill=none, font=\footnotesize, cells={anchor=west}},
]
\addplot[color=black, mark=diamond*, mark size=0.1pt, mark options={solid},
         domain=0.02:1, samples=600]
  {(1 - x)/(x*((1-x)^2 + x^2))};
\addlegendentry{$\eps_t^*=(1-t)/(t d_t)$}

\nextgroupplot[
  epsplot,
  legend style={at={(0.03,0.5)}, anchor=south west, draw=none, fill=none, font=\footnotesize, cells={anchor=west}},
]
\addplot[color=black, mark=diamond*, mark size=0.1pt, mark options={solid},
         domain=0:1, samples=600]
  {x*(1 - x)/((1-x)^2 + x^2)^2};
\addlegendentry{$\eps_t^*=\dot{\beta}_t/2=t(1-t)/d_t^2$}

\end{groupplot}
\end{tikzpicture}

\caption{Schedules $(\alpha,\beta)$ (top row) and corresponding statistically optimal diffusion scale $\eps^*$ (bottom row). These are the schedules and diffusion scales used in \cref{fig:typst_interpolant_paths}.}
\label{fig:schedules}
\end{figure}

\pagebreak
\section{Pseudocode for lazy ODE and SDE sampling}
\cref{alg:ode_sample,alg:sde_sample} give pseudocode for the algorithmic formulation of \cref{prop:ode_opt_schedule_transform,prop:sde_opt_schedule_transform} using the explicit Euler and Euler-Maruyama solver, respectively. We stress that any ODE or SDE solver can be used in practice.

\FloatBarrier

\newcommand{\SDEgray}[1]{\textcolor{black!65}{#1}}
\newcommand{\AlgPad}[1]{\STATE \vphantom{$#1$}}

\newcommand{\AlgoHeader}[2]{%
  \refstepcounter{algorithm}\label{#2}%
  \noindent\textbf{Algorithm \thealgorithm}\ #1\par
}

\begin{figure*}[h]
  \centering

  \begin{minipage}[t]{0.48\textwidth}
    \vspace{0pt}
    \hrule\vspace{0.6ex}
    \AlgoHeader{ODE-Sample}{alg:ode_sample}
    \vspace{0.6ex}\hrule\vspace{0.8ex}

    \begin{algorithmic}
      \STATE {\bfseries Input:} linear velocity $\overline{b}$, step size $\Delta t$
      \AlgPad{d_{\text{next}} \gets 1 + 2\Delta t(\Delta t-1)}
      \AlgPad{\beta_{\text{next}} \gets \left(\frac{\Delta t}{d_{\text{next}}}\right)^2}
      \AlgPad{\Delta W \gets \mathcal{N}(0,\beta_{\text{next}} I)}
      \STATE $t \gets 0$
      \STATE $x \gets \text{sample } \mathcal{N}(0, I)$

      \WHILE{$t < 1$}
        \STATE $d \gets 1 + 2t(t-1)$
        \AlgPad{\beta \gets \beta_{\text{next}}}
        \STATE $b \gets \frac{1-2t}{d}\,x \;+\; \frac{1}{\sqrt{d}}\,\overline{b}\!\left(t,\sqrt{d}\,x\right)$

        \AlgPad{t_{\text{next}} \gets t + \Delta t}
        \AlgPad{d_{\text{next}} \gets 1 + 2t_{\text{next}}(t_{\text{next}}-1)}
        \AlgPad{\beta_{\text{next}} \gets \left(\frac{t_{\text{next}}}{d_{\text{next}}}\right)^2}
        \AlgPad{\Delta W \gets \mathcal{N}(0,(\beta_{\text{next}}-\beta)I)}

        \STATE $x \gets x + \Delta t\, b$ \hfill {\footnotesize (explicit Euler step)}
        \STATE $t \gets t + \Delta t$
      \ENDWHILE
      \STATE {\bfseries Return:} $x$
    \end{algorithmic}

    \vspace{0.8ex}\hrule
  \end{minipage}
  \hfill
  \begin{minipage}[t]{0.48\textwidth}
    \vspace{0pt}
    \hrule\vspace{0.6ex}
    \AlgoHeader{SDE-Sample}{alg:sde_sample}
    \vspace{0.6ex}\hrule\vspace{0.8ex}

    \begin{algorithmic}
      \STATE {\bfseries Input:} linear velocity $\overline{b}$, step size $\Delta t$
      \STATE \SDEgray{$d_{\text{next}} \gets 1 + 2\Delta t(\Delta t-1)$}
      \STATE \SDEgray{$\beta_{\text{next}} \gets \left(\frac{\Delta t}{d_{\text{next}}}\right)^2$}
      \STATE \SDEgray{$\Delta W \gets \text{sample } \mathcal{N}\!\left(0,\beta_{\text{next}} \id\right)$}
      \STATE $t \gets \Delta t$
      \STATE $x \gets \Delta W$ \hfill {\footnotesize ($X^*_0 = b^*(0,0)=0$)}

      \WHILE{$t < 1$}
        \STATE $d \gets d_{\text{next}}$
        \STATE \SDEgray{$\beta \gets \beta_{\text{next}}$}
        \STATE $b^* \gets \frac{2}{d}\Big((1-2t)x + t\,\overline{b}\!\left(t,\frac{d}{t}x\right)\Big)$
        \STATE \SDEgray{$t_{\text{next}} \gets t + \Delta t$}
        \STATE \SDEgray{$d_{\text{next}} \gets 1 + 2t_{\text{next}}(t_{\text{next}}-1)$}
        \STATE \SDEgray{$\beta_{\text{next}} \gets \left(\frac{t_{\text{next}}}{d_{\text{next}}}\right)^2$}
        \STATE \SDEgray{$\Delta W \gets \text{sample } \mathcal{N}\!\left(0,(\beta_{\text{next}}-\beta)\id\right)$}
        \STATE $x \gets x + \Delta t\, b^* + \Delta W$ \hfill {\footnotesize (Euler-Maruyama step)}
        \STATE $t \gets t + \Delta t$
      \ENDWHILE
      \STATE {\bfseries Return:} $x$
    \end{algorithmic}

    \vspace{0.8ex}\hrule
  \end{minipage}
\end{figure*}

\FloatBarrier

\pagebreak
\section{ODE and SDE numerical integration schemes} \label{sec:app_schemes}

\FloatBarrier

Below we give numerical schemes for solving the SDE from \cref{theorem:sde} with drift $b^\eps$ and additive diffusion scale $\eps$,
\[
    \odif{X^\eps_t} = b^\eps (t, X^\eps_t) \odif{t} + \sqrt{2 \eps_t} \odif{W_t}.
\]
With $\eps \equiv 0$ the SDE degenerates to an ODE and e.g.~the Euler-Maruyama scheme becomes the explicit Euler scheme. Note that the experiments either use $\eps \equiv 0$ or $\eps = \eps^*$, and with the schedules we consider the quadratic variation from time $s$ to $t$ ($0 \leq s \leq t \leq 1$) under the latter diffusion scale, $\int_s^t 2\eps_u^* \odif{u}$, can be calculated analytically (see \cref{prop:lazy_schedule_families} and \cref{prop:linear_eps_optimal}) and is therefore used in the below schemes and in our experiments.

For a total of $N > 1$ solver steps we use a fixed step size of $\Delta t = \frac{1}{N}$. We denote the solution after $n \in \set{0, 1, \ldots, N}$ steps by $Y_n$ at time $t_n = n \Delta t$. For $n < N$ the independent Wiener process increments from time $t_n$ to $t_{n+1}$ are denoted $\Delta W_n \sample \normal(0, \Delta t \id)$ and we define $V_n \define \ps{\int_{t_n}^{t_{n+1}} 2\eps_t \odif{t}} \Delta W_n / \Delta t$. With the ``corrected'' solution at time $n$ denoted by $\tilde{Y}_n$ and initially defined as $\tilde{Y}_0 \define Y_0$ the update step from $Y_n$ to $Y_{n+1}$ is then defined as below.

\textbf{Euler-Maruyama}
    \[
        Y_{n+1} = Y_n + \Delta t b^\eps (t_n, Y_n) + V_n.
    \]
\textbf{Predictor-Corrector}
    \begin{align*}
        Y_{n+1} &= \tilde{Y}_n + \Delta t b^\eps (t_n, Y_n) + V_n, \\
        \tilde{Y}_{n+1} &= \tilde{Y}_n + \frac{1}{2} \Delta t \ps{b^\eps (t_n, Y_n) + b^\eps (t_{n+1}, Y_{n+1})} + V_n.
    \end{align*}
\textbf{Heun}
    \begin{align*}
        Y_{n+1} &= \tilde{Y}_n + \Delta t b^\eps (t_n, \tilde{Y}_n) + V_n, \\
        \tilde{Y}_{n+1} &= \tilde{Y}_n + \frac{1}{2} \Delta t \ps{b^\eps (t_n, Y_n) + b^\eps (t_{n+1}, Y_{n+1})} + V_n.
    \end{align*}

We use a variant of the predictor-corrector and Heun scheme in which the final, returned solution is $Y_N$ so that the corrected final solution $\tilde{Y}_N$ is never computed. This prevents us from ever having to evaluate the drift $b^\eps$ at $t=1$.\footnote{Evaluating the drift at $t=1$ is in fact not a problem for us. The PRX flow model we use uses linear ODE velocity parameterization and no division by zero (which can potentially occur since $\alpha_1 = 0$) is introduced in our conversion formulas from this linear velocity to the ODE velocities and SDE drifts we consider.}

Note that the ``correct'' step for the predictor-corrector scheme is equivalent to the Heun scheme. They both use an Euler-Maruyama ``predict'' step and only differ in whether the drift in this step is evaluated in $Y_n$ or $\tilde{Y}_n$. This detail means that the Heun scheme uses two drift evaluations per step while the Euler-Maruyama and predictor-corrector schemes only use one. Since evaluating this drift function in practice amount to doing an expensive forward pass in a neural network the Heun scheme is therefore roughly twice as slow as it requires $2N - 1$ function evaluations instead of $N$. To compare the three schemes fairly, we used half the amount of steps for the Heun scheme compared to the predictor-corrector and Euler-Maruyama schemes.

For the PRX model and settings used in our experiments we find that the predictor-corrector scheme generally converges faster as a function of drift evaluations than the Euler-Maruyama and Heun scheme, especially for SDE generation. This is indicated by \cref{fig:scheme_comparison_ode_flow} and \cref{fig:scheme_comparison_sde_flow} which consider sample images generated with the standard linear flow model schedule using the ODE generation ($\eps \equiv 0$) and statistically optimal SDE generation ($\eps = \eps^*$), respectively, for the same prompt ``Two husky's hanging out of the car windows''. One sees that the predictor-corrector scheme converges to the reference image at $4096$ steps (at which the Heun scheme has used $2048$ steps) slightly faster than the Euler-Maruyama and Heun scheme. All schemes more or less converge to the same reference image, as expected.

\begin{figure}[ht]
    \centering
    \includegraphics[width=\linewidth]{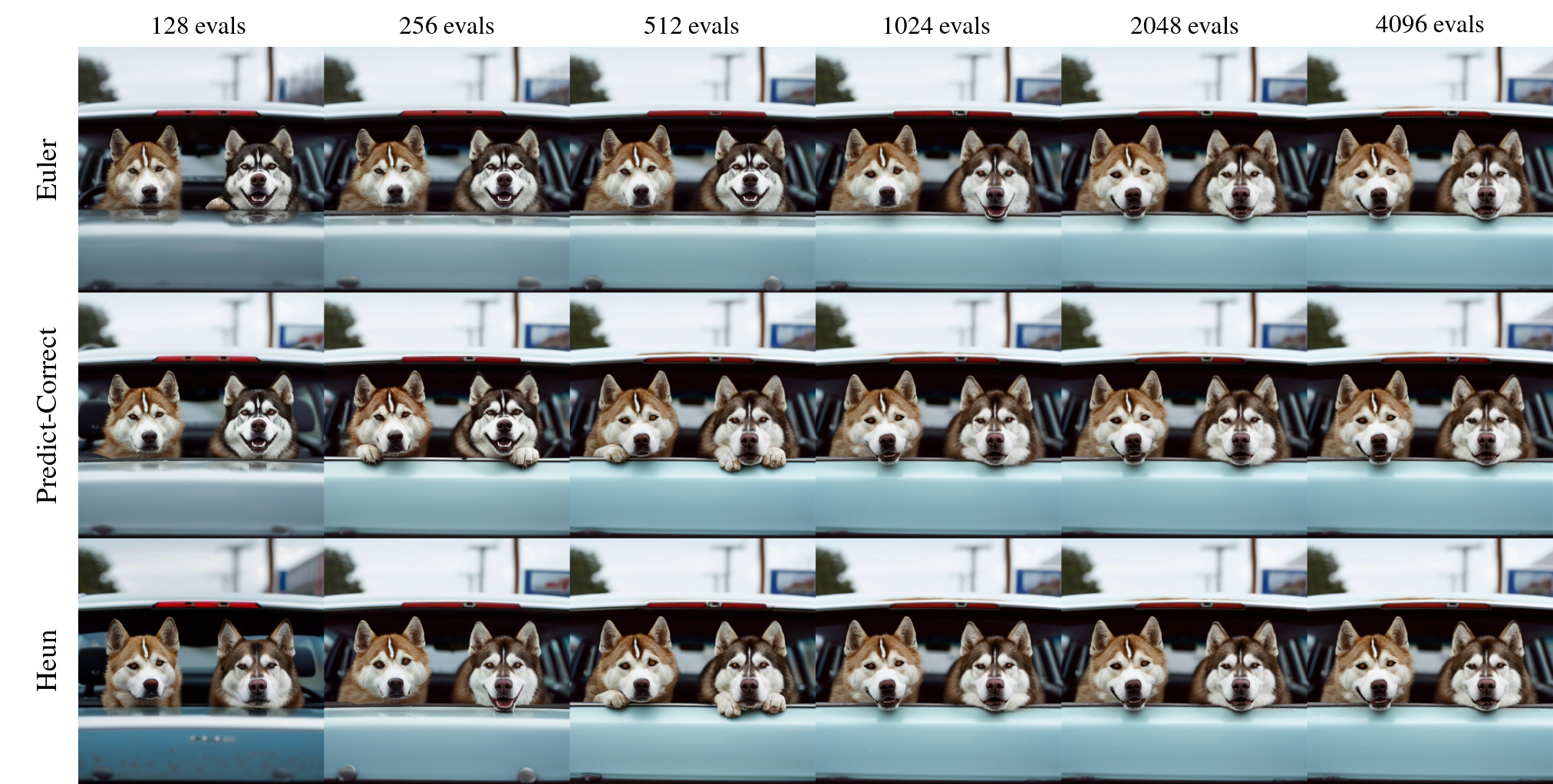}
    \caption{Sample images from the PRX flow model with varying number of steps using different integration schemes for the ODE with original linear flow model schedule. A guidance strength of $5$ with the text prompt ``Two husky's hanging out of the car windows'' was used.}
    \label{fig:scheme_comparison_ode_flow}
\end{figure}

\begin{figure}[ht]
    \centering
    \includegraphics[width=\linewidth]{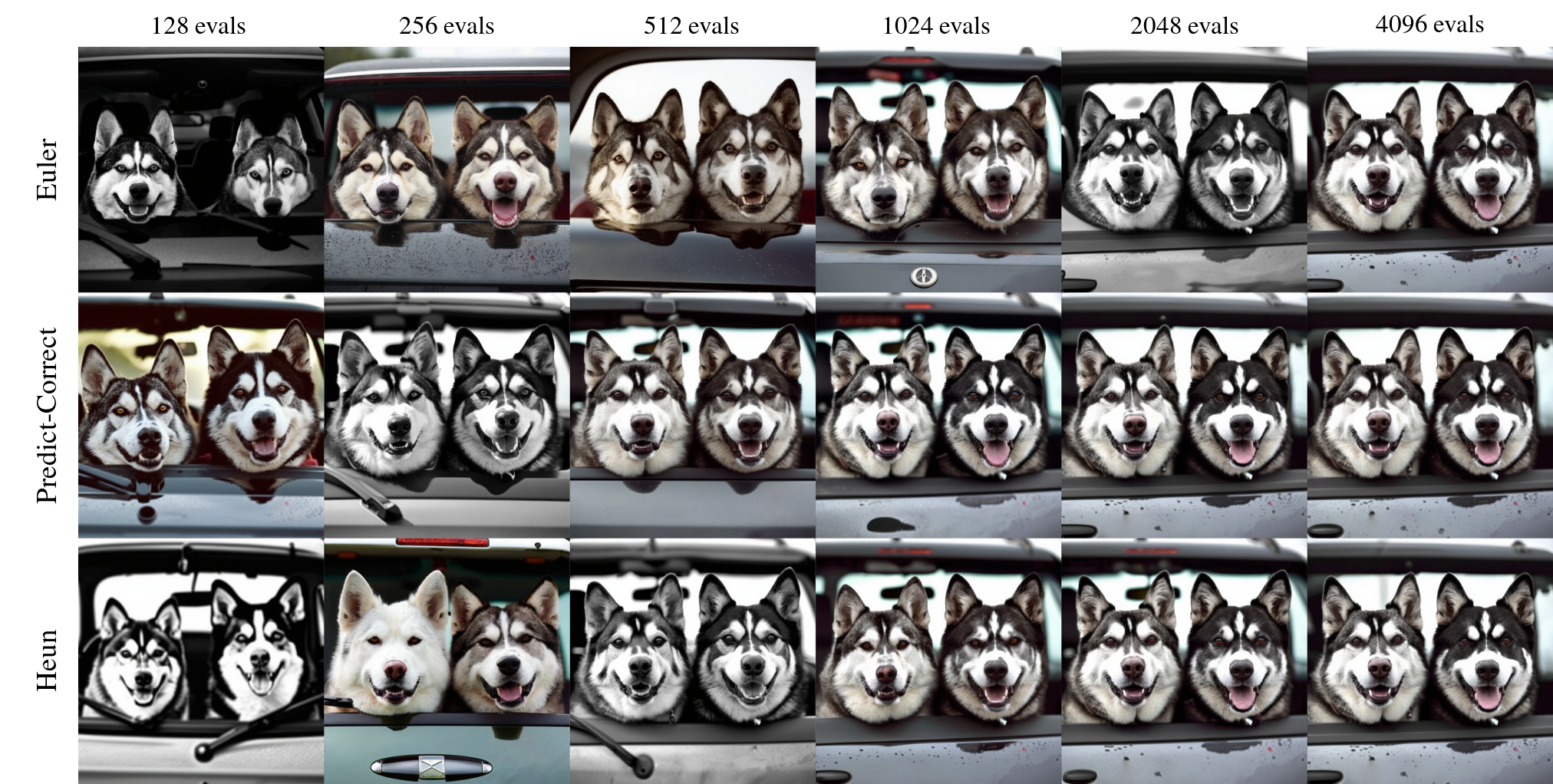}
    \caption{Same as \cref{fig:scheme_comparison_ode_flow} but for the statistically optimal SDE with original linear flow model schedule.}
    \label{fig:scheme_comparison_sde_flow}
\end{figure}

Note that the choice of integration scheme is not important for our experiments since the purpose is not to investigate which scheme leads to better image generation (faster convergence) but rather to assess the effect of transforming to the lazy schedule whilst keeping the integration scheme fixed.

\FloatBarrier

\pagebreak
\section{Further experimental results}\label{sec:app_exp_results}%
\FloatBarrier

\subsection{Solver convergence animations}
The solver convergence is better visualized as an animation than as a static image. \cref{fig:n64_samples} shows sample images for $5$ prompts generated with $64$ predictor-corrector solver steps. An animation over multiple solver steps is available at \href{\zenodourl}{this Zenodo link}. Heuristically, we find that after $64$ steps the SDE solutions are comparable to the ODE solutions in fidelity, but this is of course a subjective matter.

\begin{figure*}[ht]
    \centering
    \includegraphics[width=\linewidth]{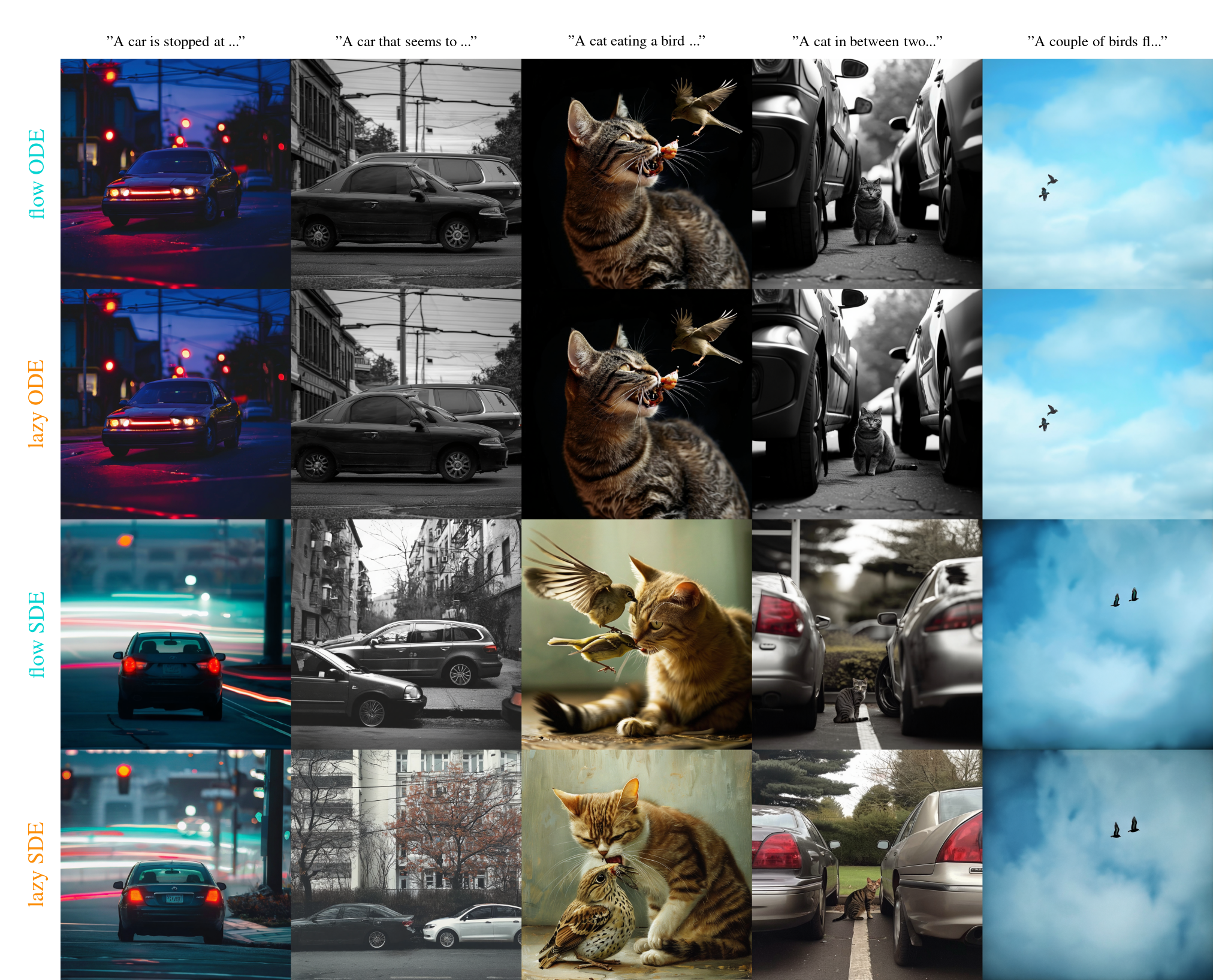}
    
    \caption{Sample images generated using $64$ predictor-corrector steps for $5$ different prompts from our experiment. Each row corresponds to whether an ODE or SDE is used and whether the original linear \textcolor{retrocyan}{flow} model schedule or the conversion to the \textcolor{retroorange}{lazy} schedule is used. An animation over all solver step amounts $\set{4, 8, \ldots, 4096}$ is available at \href{\zenodourl}{this Zenodo link}.}
    \label{fig:n64_samples}
    \vskip -0.1in
\end{figure*}

\FloatBarrier

\subsection{Equivalence of convergence}
Here we test whether the ODE and SDE solutions converge to the same reference image irrespectively of whether the flow or lazy schedule is used. For each $n_i$ amount of solver steps we calculate the RMSE between the image generated with the original linear flow model schedule against the image generated with the conversion to the lazy one using $n_i$ steps, i.e.~we calculate
\[
    \texttt{RMSE}(\texttt{img}^\text{ODE,flow}_{n_i}, \texttt{img}^\text{ODE,lazy}_{n_i})
\]
for each $i$ and similarly for the SDE. We then average over all $100$ prompts and plot this average as a function of $n_i$ for the ODE and SDE, respectively. The result is shown in \cref{fig:convergence_plot_within_step}. The figure indicates that both the ODE and SDE converge to the same reference image irrespectively of which schedule is used, but this convergence is much faster for the ODE than for the SDE. This is probably due to the fact that the flow and lazy schedule are somewhat similar for ODE generation (in particular they both satisfy $\alpha_0 = 1$), while for SDE generation the lazy schedule is a point mass schedule (see \cref{fig:schedules}). Since the flow model velocity is not learned perfectly, approximation errors make the flow-to-lazy schedule transform inexact which amplifies more for SDE generation relative to ODE generation.

\begin{figure*}[ht]
    \centering
    \includegraphics[width=\linewidth]{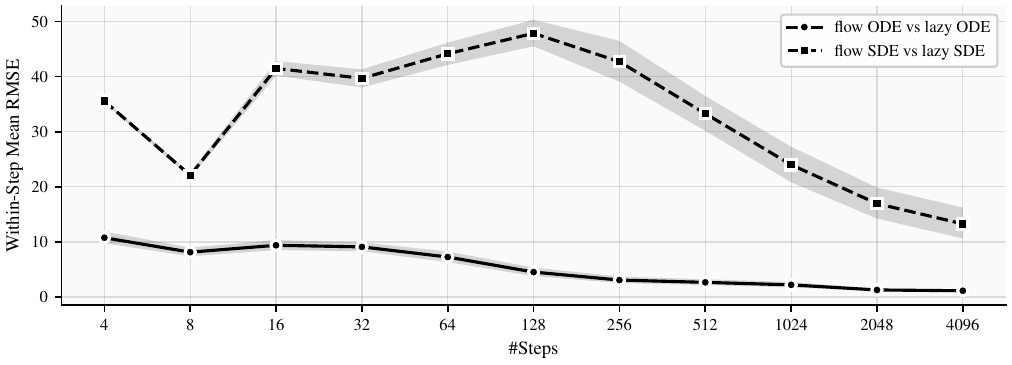}
    
    \caption{Within-step mean RMSE as a function of the number of solver steps using the predictor-corrector scheme. By ``within-step RMSE'' we mean the RMSE between the image generated under the original linear flow model schedule versus that generated using the conversion to the lazy schedule for a fixed number of solver steps, for the ODE and SDE, respectively. Shaded area indicates a 95\% confidence interval calculated by bootstrapping with $\num[group-separator={,}]{10000}$ samples. The plot indicates that the RMSE converges to zero as the number of solver steps grows, as one would expect, although the convergence is much faster for the ODE than for the SDE.}
    \label{fig:convergence_plot_within_step}
    \vskip -0.1in
\end{figure*}

\FloatBarrier

\subsection{Convergence results using the flow image as a common reference}
\begin{figure*}[ht]
    \centering
    \includegraphics[width=0.49\linewidth]{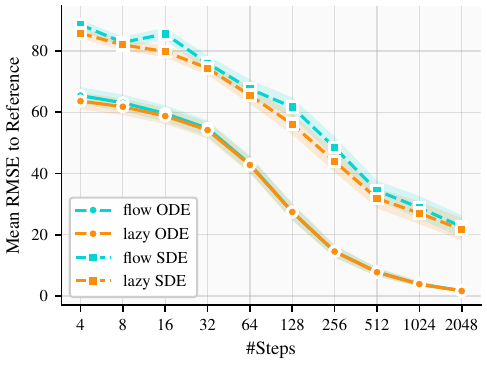}
    \includegraphics[width=0.49\linewidth]{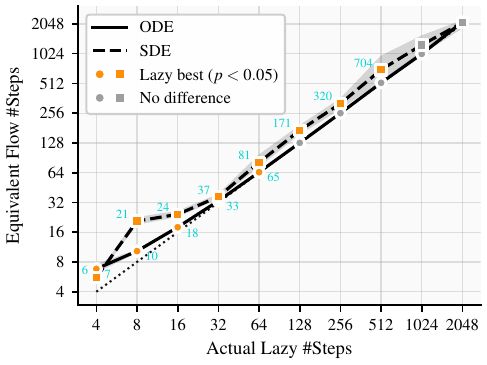}
    
    \caption{Same as \cref{fig:convergence_plot} but with $\texttt{img}_\text{reference}^\text{ODE} \define \texttt{img}^\text{ODE,flow}_{4096}$ instead of $\texttt{img}_\text{reference}^\text{ODE} \define (\texttt{img}^\text{ODE,flow}_{4096} + \texttt{img}^\text{ODE,lazy}_{4096})/2$, and similarly for the SDE.}
    \label{fig:convergence_plot_flow_ref}%
    \vskip -0.1in
\end{figure*}

Here we recreate \cref{fig:convergence_plot} but instead of defining
\[
    \texttt{img}_\text{reference}^\text{ODE} \define (\texttt{img}^\text{ODE,flow}_{4096} + \texttt{img}^\text{ODE,lazy}_{4096})/2,
\]
we instead define
\[
    \texttt{img}_\text{reference}^\text{ODE} \define \texttt{img}^\text{ODE,flow}_{4096},
\]
and similarly for the SDE. Clearly this puts the lazy schedule method at a disadvantage. However, \cref{fig:convergence_plot_flow_ref} shows that even with this disadvantage the lazy schedule converges statistically significantly faster most of the time. As one would expect, the results are slightly worse, e.g.~the lazy schedule is not statistically significantly better for SDE generation at $1024$ solver steps compared to the linear flow schedule, whereas when using the average reference as common reference it is (see \cref{fig:convergence_plot}).

\pagebreak
\section{COCO captions} \label{sec:app_coco_captions}
The COCO validation dataset consists of images with five human generated captions each. We extract the first caption from the first $100$ images and remove any potential trailing period. We use the resulting $100$ sentences as text prompts for the PRX flow model. These are shown below sorted alphabetically (ASCII-wise).

\FloatBarrier

\begin{Verbatim}[numbers=left,xleftmargin=5mm]
A bathroom sink with toiletries on the counter
A bathroom with a poster of an ugly face above the toilette
A bathroom with a sink and shower curtain with a map print
A bathroom with a toilet, sink, and shower
A bathroom with a traditional toilet next to a floor toilet
A beautiful dessert waiting to be shared by two people
A bicycle replica with a clock as the front wheel
A bike leaning against a sign in Scotland
A black Honda motorcycle parked in front of a garage
A black and white photo of an older man skiing
A black and white toilet with a plunger
A black cat is inside a white toilet
A boat that looks like a car moves through the water
A brown and black horse in the middle of the city eating grass
A brown purse is sitting on a green bench
A car is stopped at a red light
A car that seems to be parked illegally behind a legally parked car
A cat eating a bird it has caught
A cat in between two cars in a parking lot
A couple of birds fly through a blue cloudy sky
A cute kitten is sitting in a dish on a table
A delivery truck with an advertisement for Entourage
A dirt bike rider doing a stunt jump in the air
A dog and a goat with their noses touching at fence
A dog sitting between its masters feet on a footstool watching tv
A door with a sticker of a cat door on it
A dual sink vanity with mirrors above the sinks
A few items sit on top of a toilet in a bathroom stall
A fighter jet is flying at a fast speed
A fireplace with a fire built in it
A gas stove next to a stainless steel kitchen sink and countertop
A green car on display next to a busy street
A group of people preparing food in a kitchen
A group of people riding mopeds in a busy street
A kitchen is shown with wooden cabinets and a wooden celling
A large U.S Air Force plain sits on an asphalt ramp
A large passenger airplane flying through the air
A large passenger jet taking off from an airport
A little girl in a public bathroom for kids
A long empty, minimal modern skylit home kitchen
A man getting a drink from a water fountain that is a toilet
A man in a wheelchair and another sitting on a bench that is overlooking the water
A man is sitting on a bench next to a bike
A man sits with a traditionally decorated cow
A man standing in a kitchen while closing a cupboard door
A parade of motorcycles is going through a group of tall trees
A person holding a skateboard overlooks a dead field of crops
A person walking in the rain on the sidewalk
A photograph of a kitchen inside a house
A picture advertising Arizona tourism in an airport
A picture of a man playing a violin in a kitchen
A pinewood and green modern themed kitchen area
A public bathroom with a bunch of urinals
A random plane in the sky flying alone
A room with blue walls and a white sink and door
A shot of an elderly man inside a kitchen
A small car is parked in front of a scooter
A small child climbs atop a large motorcycle
A small closed toilet in a cramped space
A tiny bathroom with only a toilet and a shelf
A toilet sits next to a window and in front of a shower
A trio of dogs sitting in their owner's lap in a red convertible
A variety of pots are stored in a nook by a fireplace
A woman is walking a dog in the city
An airplane with its landing wheels out landing
An all white kitchen with an electric stovetop
An office cubicle with four different types of computers
An office cubicle with multiple computers in it
An old toilet with a hello kitty cover top
An old-fashioned green station wagon is parked on a shady driveway
Fog is in the air at an intersection with several traffic lights
Half of a white cake with coconuts on top
Jet liner flying off into the distance on an overcast day
Little birds sitting on the shoulder of a giraffe
Man in motorcycle leathers standing in front of a group of bikes
Office space with office equipment on desk top
Pedestrians walking down a sidewalk next to a small street
People are waiting for the bus near a bus stop
Posted signs point the way through a parking garage
Riding a motorcycle down a street that has no one else on it
Rows of motor bikes and helmets in a city
Several motorcycles riding down the road in formation
The airplane is on the runway with two young men standing by
The back door with a window in the kitchen
The sign of a restaurant in the outside of the store
The telephone has a banana where the receiver should be
The woman in the kitchen is holding a huge pan
The young man is stirring his pot of food with a wooden spoon
This is an open box containing four cucumbers
Two husky's hanging out of the car windows
Two women waiting at a bench next to a street
Women father around a desk and machinery in a factory
a bike leaning on a metal fence next to some flowing water
a counter with vegetables, knife and cutting board on it
a man sleeping with his cat next to him
a man with a bike at a marina
a modern flush toilet in a bathroom with tile
a small toilet stall with a toilet brush and 3 rolls of toilet paper
an airport with one plane flying away and the other sitting on the runway
four urinals in a public restroom with a window
\end{Verbatim}

\FloatBarrier


\end{document}